\documentclass{article}
\usepackage{arxiv}

\usepackage{authblk}

\usepackage{natbib} 

\usepackage{url}            
\usepackage{amsfonts}       
\usepackage{tikz} 
\usetikzlibrary{calc}
\usetikzlibrary{arrows}
\usepackage{amsmath,mathpazo,mathtools,bm}
\usepackage{amsthm}
\usepackage{amssymb}
\usepackage{dsfont} 
\usepackage{amssymb}
\usepackage{pgfplots}
\usetikzlibrary{positioning}
\usepackage{mathabx}

\pgfplotsset{compat=1.18}
\usepackage{subcaption}
\usepackage{algorithm}
\usepackage[noend]{algorithmic}

\usepackage{enumitem}

\usetikzlibrary{arrows.meta}

\usepackage{thm-restate}

\usepackage{ciphod}
\usepackage{mathbbol}
\usepackage{subcaption}
\usepackage{natbib}

\usepackage{xparse}

\RenewDocumentEnvironment{subfigure}{m}
 {%
  \IfNoValueTF{#1}
    {\begin{minipage}{\linewidth}}
    {\IfBlankTF{#1}
      {\begin{minipage}{\linewidth}}
      {\begin{minipage}{#1}}}%
 }
 {\end{minipage}}

\DeclareSymbolFontAlphabet{\mathbbl}{bbold}

\newcommand{\indep}{\rotatebox[origin=c]{90}{$\models$}}
\newcommand{\nindep}{\not\!\!\rotatebox[origin=c]{90}{$\models$}}

\newtheorem{definition}{Definition}
\newtheorem{theorem}{Theorem}
\newtheorem{lemma}{Lemma}

\newcommand{\arrowdoublebar}[1][1em]{%
  \mathrel{%
    \tikz[baseline]{%
      \draw[-](0,.58ex)--(#1,.58ex);%
      \draw($(#1,.58ex)+(0pt,3pt)$)--($(#1,.58ex)+(0pt,-3pt)$);%
      \draw($(#1,.58ex)+(-2pt,3pt)$)--($(#1,.58ex)+(-2pt,-3pt)$);%
    }%
  }%
}

\title{Local Markov Equivalence for PC-style Local Causal Discovery and Identification of Controlled Direct Effects}

\author[1]{Timothée Loranchet}
\author[1]{Charles K. Assaad}
\affil[1]{Sorbonne Université, INSERM, Institut Pierre Louis d’Epidémiologie et de Santé Publique, F75012, Paris, France}   

\date{}

\begin{document}
\maketitle

\begin{abstract}
Identifying controlled direct effects (CDEs) is crucial across numerous scientific domains. While existing methods can identify these effects from causal directed acyclic graphs (DAGs), the true DAG is often unknown in practice. Essential graphs, which represent a Markov equivalence class of DAGs characterized by the same set of conditional independencies, provide a more practical and realistic alternative, and the PC algorithm is one of the most widely used method to learn them using conditional independence tests.
However, learning the full essential graph is computationally intensive and relies on strong, untestable assumptions. In this work, we adapt the PC algorithm to recover only the portion of the graph needed for identifying CDEs.
In particular, we introduce the local essential graph (LEG), a graph structure defined relative to a target variable, and present LocPC, an algorithm that learns the LEG using solely local conditional independence tests. Building on this, we develop LocPC-CDE, which extracts precisely the portion of the LEG that is both necessary and sufficient for identifying a CDE. Compared to global methods, our algorithms require less conditional independence tests and operate under weaker assumptions while maintaining theoretical guarantees. We illustrate the effectiveness of our approach on synthetic and real data.
\end{abstract}

\section{INTRODUCTION}
\label{sec:intro}

Understanding controlled direct effects~\citep{Pearl_2000,Pearl_2001} (CDEs) is fundamental to causal inference across a wide array of scientific fields, including public health~\citep{Vansteelandt_2009,Vanderweele_2010} and industries~\citep{Assaad_2023}. CDEs quantify how changes in one variable influence another independently of any mediating pathways, offering valuable insight into mechanisms of action and informing targeted interventions.
For instance,  suppose an epidemiological study on the effect of physical exercise on cardiovascular health, where estimating the CDE reveals that exercise improves heart outcomes even without weight loss.
This insight is critical as it shows that interventions promoting exercise should be encouraged even in individuals who do not lose weight, shifting public health messaging and clinical advice to focus on exercise benefits beyond weight control.

Numerous tools exist for identifying  CDEs when the underlying causal structure is known and can be represented as a directed acyclic graph (DAG). However, in many practical scenarios, the true DAG is unknown. Instead, one can often recover, under the causal sufficiency and the faithfulness assumptions~\citep{Spirtes_2000}, an essential graph~\citep{Andersson_1997}, which encodes a Markov equivalence class of DAGs consistent with the observed conditional independencies.
While global causal discovery methods such as the PC algorithm~\citep{Spirtes_2000} aim to recover the entire essential graph, they are often computationally intensive and require large amounts of data. 
In addition,  they often fail in real applications  due to the violation of the strong and untestable assumptions they rely on~\citep{Uhler_2012,Andersen_2013,Assaad_2022_a,Ait_Bachir_2023}.

This motivates the development of local causal discovery methods that concentrate solely on the relevant subgraph surrounding the variables of interest. 
Restricting the discovery process to a local neighborhood, would not only reduce computational complexity but also enhance robustness, while still yielding valid and actionable causal insights.
For instance, \cite{Gao_2015} proposed an algorithm to discover the immediate neighborhood of a target variable. 
Other works proposed algorithms aimed at discovering parts of the essential graph sufficient for identifying total effects using only local information \citep{Maathuis_2008,Malinsky_2016,Gupta_2023}. 
Nevertheless, the problem of local causal discovery targeted for identifying a CDE remains largely unaddressed. 
Moreover, the characterization of the class of graphs recoverable through local information has received little attention in the literature.
Finally, all existing theoretically founded local causal discovery algorithms adopt a strategy different from the PC algorithm, and for good reason: a naive adaptation of the PC algorithm to the local setting can introduce spurious edges, a well-documented issue in the literature~\citep{Li2015,schubert2025localcausaldiscoverystatistically}, whose full characterization has, however, been largely overlooked.

In this paper, we focus on PC-style local causal discovery and provide a complete characterization of the spurious edges inferred by such an approach.
We specifically demonstrate that it arises only due to certain types of inducing paths~\citep{Spirtes_2000}. 
Moreover, we characterize the class of graphs that share a specific notion of locality around a target variable $Y$ 
that dependents on a targeted neighborhood hop distance. We show that this class can be represented by a local essential graph (LEG) and introduce a PC-style local causal discovery algorithm, called \textbf{LocPC}, for recovering the LEG from data, requiring weaker assumptions than those needed for recovering the full essential graph. Furthermore, we demonstrate that a naive application of \textbf{LocPC} can serve for local causal discovery aimed at identifying a CDE. Finally, we develop an improved version of the algorithm, called the \textbf{LocPC-CDE} algorithm that optimally discovers only the part of the LEG that is both necessary and sufficient for CDE identification. 

The remainder of the paper is organized as follows: Section~\ref{sec:relatedworks} provides an overview of the related works.
Section~\ref{sec:setup} introduces preliminaries.
Section~\ref{sec:characterization} presents a characterization of all graphs that have the same local information. 
Section~\ref{sec:LCD} presents the \textbf{LocPC} algorithm. Section~\ref{sec:LCD-CDE} starts by showing how a naive application of \textbf{LocPC} can be used for identifying a CDE and then presents the \textbf{LocPC-CDE} algorithm. 
In Section~\ref{sec:exp}, we evaluate the proposed algorithms on synthetic and real data. 
Finally, Section~\ref{sec:conc} concludes the paper.
All proofs are deferred to Appendix~\ref{appendix:proofs}.

\section{Related Work}
\label{sec:relatedworks}

Initially, local causal discovery have been developed 
for identifying a target node’s adjacency or Markov blanket~\citep{Li2015,aliferis-2003,Tsamardinos-2003,JMLR:v11:aliferis10a}. 
However, such approaches do not orient edges. Our work focuses on local causal discovery for \emph{CDE identification}, which requires orienting edges around the outcome. The CMB algorithm~\citep{Gao_2015} and MBbyMB algorithm~\citep{wang2014discovering} orients the adjacency of a target node to identify its direct causes and effects, but it relies on Markov Blanket discovery and is not PC-style. This approach has been extended to cases where causal sufficiency is relaxed~\citep{xie_2024,li_2025}. There exists also hybrid methods, such as LDECC~\citep{Gupta_2023}, that use a Markov blanket strategy as well as a PC-style strategy to identify a total effect. Nevertheless, LDECC can be adapted for CDE by targeting the treatment rather than the outcome. It is important to note that none of these works formally characterize the structural information extractable by a local PC-style algorithms. Moreover, it has been recently shown that several of these algorithms are not sound~\citep{schubert2025localcausaldiscoverystatistically}. We fill this gap by showing that this characterization is non-trivial, that neglecting it may yield spurious edges, and by optimizing PC-style local discovery for CDE identification.
Along a different line of research, some methods aim to identify causal effects using only local information available in the essential graph. However, unlike the focus of this paper, these approaches assume that the essential graph is fully known~\citep{Maathuis_2008}. This line of work has also been extended to cases where the causal sufficiency assumption is relaxed~\citep{Malinsky_2016}. All these methods are considered as beyond the scope of this paper.


\section{Preliminaries}
\label{sec:setup}

We use capital letters $(Z)$ for variables, bold letters $(\mathbb Z)$ for sets, and $|\mathbb Z|$ for their size. 

In this work, we rely on the framework of Structural Causal Models (SCMs) as introduced by \citep{Pearl_2000}. Formally, an SCM $\mathcal{M}$ is defined as a $4$-tuple $(\mathbb{U},\mathbb{V},\mathbb{F},P(\mathbb{U}))$, where $\mathbb{U}$ denotes a set of exogenous variables, with $P(\mathbb{U})$ denoting their joint distribution, and $\mathbb{V}$ denotes a set of endogenous variables. The set $\mathbb{F}$ contains causal mechanisms, each determining an endogenous variable from a corresponding exogenous variable and a subset of other endogenous variables, usually referred to as direct causes or parents. 
We assume that the SCM induces a directed acyclic graph (DAG) $\mathcal G = (\mathbb V, \mathbb E)$ consisting of a set of vertices $\mathbb V$ and directed edges $\mathbb E \subseteq \mathbb V \times \mathbb V$.
Additionally, we assume  \textbf{causal sufficiency}, meaning that all exogenous variables are mutually independent and each influences only a single endogenous variable. 
In $\mathcal G$, a parent of $W_l \in \mathbb V$ is any $W_m \in \mathbb V$ such that $W_m \to W_l$ is in $\mathbb E$. The set of parents of $W_l$ is denoted $Pa(W_l,\mathcal{G})$. 
The descendants of $W_l$,$De(W_l,\mathcal{G})$, are those reachable by a directed path from $W_l$. The neighbors of $W_l$, $Ne(W_l,\mathcal{G})$, are all variables connected to $W_l$ in $\mathcal G$.  
The $h$-hop neighborhood of a target variable $Y$, denoted $NeHood(Y,h,\mathcal G)$ is the set of nodes $\mathbb{Z}\subseteq\mathbb{V}$ such that the shortest path between $Y$ and any node in $\mathbb{Z}$ is less than or equal to $h$.
A node $W_l$ in $\mathcal G$ is considered a collider on a path $p$ if there exists a subpath $W_k \rightarrow W_l \leftarrow W_m$ within $p$. In this context, we will interchangeably refer to the triple $W_k \rightarrow W_l \leftarrow W_m$ and the node $W_l$ as the collider. Furthermore, $W_k \rightarrow W_l \leftarrow W_m$ is termed an unshielded collider (UC) if $W_k$ and $W_m$ are not adjacent.
A path $p$ is said to be blocked by a set $\mathbb{Z}$ if and only if  1)
$p$ contains a non-collider triple (\ie, $W_k\rightarrow W_l \rightarrow W_m$ or $W_k\leftarrow W_l \leftarrow W_m$ or $W_k\leftarrow W_l \rightarrow W_m$) such that the middle node ($W_l$) is in $\mathbb{Z}$, or 2) $p$ contains a collider (\ie, $W_k\rightarrow W_l \leftarrow W_m$) such that $(De(W_l,\mathcal{G})\cup\{W_l\})\cap \mathbb{Z}=\emptyset$. 
Two nodes $W_l$ and $W_m$ are $d$-separated by $\mathbb{Z}$, denoted $(W_l\indep W_m\mid \mathbb{Z})_\mathcal G$, if and only if all paths between $W_l$ and $W_m$  are blocked by $\mathbb{Z}$~\citep{Pearl_2000}. The $d$-separation between $W_l$ and $W_m$ by $\mathbb{Z}$ implies that $W_l$ and $W_m$ are independent conditional on $\mathbb{Z}$, denoted $(W_l\indep  W_m\mid \mathbb{Z})_\mathcal P$, in every distribution $\mathcal{P}$ that is compatible with $\mathcal G$.
Multiple DAGs can encode the same set of $d$-separations, forming what is known as a Markov equivalence class (MEC). Under \textbf{causal sufficiency}, any two DAGs within the same MEC share both the same adjacencies and the same UCs~\citep{Verma_1990}. This structural similarity allows every MEC to be uniquely represented by an essential graph, also known as a CPDAG~\citep{Chickering_2002,Andersson_1997,Meek_1995}, denoted $\mathcal{C}$.
An essential graph captures all common adjacencies and encodes edge orientations that are invariant across all DAGs in the MEC.
Specifically, a directed edge $W_l \rightarrow W_m$ in the essential graph implies that this orientation is present in every DAG in the MEC. In contrast, an undirected edge $W_l - W_m$ signals ambiguity: some DAGs contain $W_l \rightarrow W_m$ while others contain $W_l \leftarrow W_m$.
The notions of parents and neighbors naturally extend to essential graphs. Specifically, a node $W_l$ is a parent of $W_m$ if $W_l \rightarrow W_m$ in the essential graph, meaning it is a parent in all DAGs consistent with that graph. We denote these sets by $Pa(W_l, \mathcal{C})$ and $Ne(W_l, \mathcal{C})$.

Essential graphs are particularly valuable because they can be learned from observational data under \textbf{causal sufficiency} and an additional key assumption, called the faithfulness 
assumption~\citep{Spirtes_2000} 
which posits that all the conditional independencies observed in the data distribution correspond to $d$-separation relations in the true underlying causal DAG.
Under these assumptions, structure learning algorithms can recover the essential graph corresponding to the true  DAG's MEC. One of the most well-known algorithms for this purpose is the PC algorithm~\citep{Spirtes_2000}. 
In a nutshell, the PC algorithm uses conditional independence (CI) tests to infer the skeleton of the graph, meaning it removes edges between two nodes $W_l$ and $W_m$ if there exists a set $\mathbb{Z}$ such that $(W_l \indep W_m \mid \mathbb{Z})_\mathcal P$. Then, for each unshielded triple $W_k - W_l - W_m$ in the skeleton, it identifies it as an UC $W_k \rightarrow W_l \leftarrow W_m$ if the middle node $W_l$ was not included in the conditioning set that yielded the independence between  $W_k$ and $W_m$. Finally, the algorithm orients as many other edges as possible using Meek’s rules~\citep{Meek_1995}.

In this paper, 
we concentrate on the controlled direct effect (CDE) of treatment variable $X$ on a target variable $Y$ in a non-parametric setting~\citep{Pearl_2000,Pearl_2001}, denoted as $CDE(x,x',Y)$ and formally expressed as  
$E(Y\mid do(x), do(pa_{Y\backslash X})) - E(Y\mid do(x'), do(pa_{Y\backslash X}))$,     where $pa_{Y\backslash X}$ stands for any realization of the parents of $Y$ in $\mathcal G$ excluding $X$, and $do(\cdot)$ operator represents an intervention.
A $CDE(x,x',Y)$ is said to be identifiable if it can be uniquely computed from the positive observed distribution~\citep{Pearl_2000}. Causal graphs are invaluable for identifying causal effects in general. Specifically, it has been shown that under \textbf{causal sufficiency}, the $CDE(x,x',Y)$ is always identifiable from a DAG. It is also identifiable from an essential graph if and only if there is no undirected edge connected to $Y$ \cite[Theorem 5.4]{Flanagan_2020}.

\section{Characterization of Local Markov Equivalence}
\label{sec:characterization}


The classical PC algorithm succeeds in global causal discovery because, when testing conditional independence between two variables $Y$ and $A$, it searches for conditioning sets among the neighbors of \emph{both} variables. 
Adapting the PC algorithm to local causal discovery is far from straightforward, a difficulty already emphasized in the literature~\citep{Li2015,schubert2025localcausaldiscoverystatistically}. Without loss of generality, we illustrate the issue by focusing on the local discovery of the neighbors of a single target variable $Y$ (our results extend naturally to the discovery of the neighbors of any target set of variables).
A seemingly natural adaptation (already explored in prior work~\citep{Gupta_2023}) is to \textbf{localize this procedure by applying PC only around the target variable $Y$ by testing conditional independence between $Y$ and another variable $A$ using conditioning sets \emph{only} among the neighbors of $Y$. }
Unfortunately, this restriction can lead to add spurious edges: conditioning solely on $Y$'s neighbors may not be sufficient to remove all variables that are not true causal neighbors of $Y$. An example illustrating this failure is provided in  Appendix~\ref{appendix:spurious}. 
Although the limitations of local PC-style adaptations have already been noted, no prior work has identified precisely which variables are guaranteed to remain spuriously adjacent to $Y$. 
In order to provide this characterization, we start by introducing $\mathbb{C}_s(Y,\mathcal G)$ which denotes the set of nodes still adjacent to $Y$ at iteration $s$ in a PC-style procedure to find all nodes adjacent to a target variable in a DAG $\mathcal{G}$. 
Since PC starts with a fully connected graph,  $
\mathbb C_0(Y, \mathcal G) := \mathbb V \setminus \{Y\}$. Then, $\forall s \ge 1$, the set $\mathbb C_s(Y, \mathcal G)$, can be recursively  defined as:
\begin{equation*}    
\label{eq:set_of_neighbors_at_iteration_s}
\mathbb C_s(Y, \mathcal G) := 
\big\{
A \in \mathbb C_{s-1}(Y, \mathcal G) \;\big |\;
\nexists\, \mathbb Z \subseteq \mathbb C_{s-1}(Y, \mathcal G)\setminus\{A\}
:\; [|\mathbb Z| = s]
\land [\; (Y \indep A \mid \mathbb Z)_{\mathcal G}]
\big\}.
\end{equation*}
In a nutshell, for any node $Y$ of interest under PC-style local discovery, all edges adjacent to $Y$ are first added by the algorithm, then pruned whenever a separating set of size $s$ is found among nodes adjacent to $Y$ at iteration $s-1$, iterating on $s$ from $s=0$.
Spurious nodes are non-neighbors that remain adjacent at the end of the local PC procedure (which contains at most $s_{max}=|\mathbb V|-2$ iterations). More formally, given $\mathbb{C}_s(Y, \mathcal{G})$, the set of spurious neighbors can be defined as follows:


\begin{definition}[Spurious neighbors]
\label{def:sne}
Let $\mathcal G := (\mathbb V, \mathbb E)$ be a DAG and let $Y \in \mathbb V$, the \emph{spurious neighbors} of $Y$ in $\mathcal G$ are defined as
$SNe(Y, \mathcal G) := 
\mathbb C_{|\mathbb V|-2}(Y, \mathcal G) \setminus Ne(Y, \mathcal G).$
\end{definition}

So far, we have focused on a \textit{single} target variable $Y$, but the results can be readily extended to a set of variables. In this paper, we mainly consider the \textit{target neighborhood} of hop $h > 0$ around $Y$.
In this setting, spurious neighbors can only occur between a node $D\in NeHood(Y,h,\mathcal G)\backslash NeHood(Y,h-1,\mathcal G)$ and a node outside the neighborhood, as will be rigorously proven later in the paper.
Interestingly, any spurious neighbor necessarily lies at the endpoint of a very specific type of path, which we call a \emph{descendant inducing paths}.

\begin{definition}[Descendant Inducing Path (DIP)]
\label{def:dip}
Let $\mathcal G=(\mathbb V,\mathbb E)$ be a DAG and $A,B\in\mathbb V$ \st $B\notin Ne(A,\mathcal G)$. Let $\pi$ be a path between $A$ and $B$, $V(\pi)$ the set of nodes on $\pi$, and $C(\pi)$ the set of colliders on $\pi$. Define $\mathbb L:=(L_1,\dots,L_k)$ as the nodes of $\{A,B\}\cup \{Ne(A,\mathcal G)\cap V(\pi)\}$ ordered according to their appearance along $\pi$. The path $\pi$ is a \emph{Descendant Inducing Path} (DIP) relative to $\mathbb L$ if $C(\pi)=\mathbb L\setminus \{A,B\}$ and $L_{i+1}\in De(L_i,\mathcal G)$ for all $i=1,\dots,k-1$.
\end{definition}

DIPs constitute a restricted subclass of inducing paths~\citep{Spirtes_2000}. More details and examples about DIPs are presented in Appendix~\ref{appendix:spurious}.
We call \textit{descendant inducing neighbors} of a node $D\in NeHood(Y,h,\mathcal{G})$, denoted $DINe(D,\mathcal G)$, the set of nodes $A$ 
such that there exists a DIP from $D$ to $A$ in $\mathcal G$. The following theorem establishes the relation between $SNe(D,\mathcal G)$, $DINe(D,\mathcal G)$, and $De(D,\mathcal G)$.

\begin{restatable}[Spurious neighbors are descendant inducing neighbors]{theorem}{mypropositionone}
\label{prop:spurious}
Let $\mathcal G := (\mathbb V, \mathbb E)$ be a DAG, let $Y \in \mathbb V$, and let $D\in NeHood(Y,h,\mathcal{G})$. Then $SNe(D,\mathcal G)\subseteq DINe(D,\mathcal G)\subset De(D,\mathcal G).$
\end{restatable}

Given the structural complexity of DIPs and descendant-inducing neighbor, spurious neighbors should be rare. For an illustration of Theorem~3, see Appendix~\ref{appendix:spurious}.

%
%
Having introduced and characterized spurious neighbors, which is essential for a rigorous local PC-style causal discovery, we now characterize the equivalence class of all DAGs sharing the $d$-separations explored by a local PC-style algorithm. We refer to this class as the Local Markov Equivalence Class (LMEC).

\begin{definition}[{Local Markov equivalence class, LMEC}]
\label{def:lmec}
Consider a DAG $\mathcal{G}=(\mathbb{V},\mathbb{E})$, a target vertex $Y\in \mathbb{V}$ and an integer $h$.  
We define the \emph{local Markov equivalence class} of $\mathcal{G}$ relative to a vertex $Y$ and its $h$-hop neighborhood, denoted by $LMEC(Y, h,\mathcal{G})$, as the set of graphs such that $\forall \mathcal{G}_i\in LMEC(Y, h,\mathcal{G})$, $\forall D \in NeHood(Y,h,\mathcal G)$, $\forall s \in  \{1,...,|\mathbb V|-2\}$, $\forall W\in \mathbb C_{s-1}(D,\mathcal G)$:
$
\forall \mathbb S\subseteq \mathbb C_{s-1}(D,\mathcal G)\setminus \{W\}, |\mathbb S| = s:(D\indep W\mid \mathbb S)_\mathcal G\iff (D\indep W\mid \mathbb S)_{\mathcal G_i}
$.
\end{definition}

Theorem~\ref{theorem:lmec} derives graphical characterization of all DAGs within the same LMEC.


\begin{restatable}[Structural characteristics of LMEC's DAGs]{theorem}{mytheoremone}
\label{theorem:lmec}
Consider a DAG $\mathcal{G}=(\mathbb{V},\mathbb{E})$ and $Y\in \mathbb{V}$. We have the following $\forall \mathcal G_i,\mathcal G_j \in LMEC(Y,h,\mathcal{G})$:
    \begin{enumerate}
        \item \textbf{Same $h$-Neighborhood:} $NeHood(Y, h, \mathcal{G}_i)=NeHood(Y, h, \mathcal{G}_j),$
        \item \textbf{Same local skeleton:} $\forall D\in NeHood(Y, h, \mathcal{G}_i):Ne(D,\mathcal{G}_i)\cap NeHood(Y,h,\mathcal G_i) = Ne(D,\mathcal{G}_j)\cap NeHood(Y,h,\mathcal G_i),$
        \item \textbf{Same outside neighbors:} $\forall D\in NeHood(Y, h, \mathcal{G}_i):Ne(D,\mathcal{G}_i)\cup SNe(D,\mathcal{G}_i) = Ne(D,\mathcal{G}_j)\cup SNe(D,\mathcal{G}_j),$
     \item \textbf{Same local UCs:} A UC involving the triplet $(D_1,D_2,D_3)$ with $D_1, D_2,D_3 \in NeHood(Y,h,\mathcal{G}_i)$ appears in $\mathcal G_i$ if and only if the same UC appears in $\mathcal{G}_j$,
    \item \textbf{Same "no non-collider":} Let $D \in NeHood(Y,h,\mathcal{G}_i)$, $A \in Ne(D,\mathcal{G}_i) \cup SNe(D,\mathcal{G}_i)$, and define 
$\mathbb{W}_D = \mathbb{V} \setminus \Big\{ NeHood(Y,h,\mathcal{G}_i) \cup Ne(D,\mathcal{G}_i) \cup SNe(D,\mathcal{G}_i) \Big\}$. Then, $\nexists W \in \mathbb{W}_D$, such that the triple $(D,A,W)$ is a non-collider triple in $\mathcal{G}_i$ if and only if $\nexists W \in \mathbb{W}_D$ such that the triple $(D,A,W)$ is not a non-collider triple $\mathcal{G}_j$.


    \end{enumerate} 
\end{restatable}

Now, we introduce a new graphical representation, which we call \emph{local essential graph} (LEG), that represents all graphs in a given LMEC.



\begin{definition}[Local essential graph, LEG]
\label{def:leg}
Let $\mathcal{G} = (\mathbb{V}, \mathbb{E})$ be a DAG, and let $Y \in \mathbb{V}$ be a vertex of interest.  
The \emph{local essential graph} (LEG) associated with $LMEC(Y, h, \mathcal{G})$, denoted by $\mathcal{L}^{Y,h} = (\mathbb{V}, \mathbb{E}^{Y,h})$, is defined as the partially directed acyclic graph over $\mathbb{V}$ satisfying the following conditions for all $D \in NeHood(Y,h,\mathcal G)$, $A \in NeHood(Y,h+1,\mathcal{G})$, and $W \in \mathbb{V}$:
\begin{enumerate}
    \item \textbf{Local undirected edge:} $(D_1 - D_2) \in \mathbb{E}^{Y,h}$ if and only if, $\forall \mathcal G_i \in LMEC(Y,h,\mathcal{G})$, either $D_1 \rightarrow D_2$ or $D_1 \leftarrow D_2$ is present, and $\exists \mathcal G_i, \mathcal G_j \in LMEC(Y,h,\mathcal{G})$ such that $D_1 \rightarrow D_2$ appears in $\mathcal G_i$ and $D_1 \leftarrow D_2$ appears in $G_j$.
    \item \textbf{Outside undirected edge:} $(D - A) \in \mathbb{E}^{Y,h}$ if and only if, $\forall \mathcal G_i\in LMEC(Y, h, \mathcal{G})$, $A\in Ne(D,\mathcal G_i)\cup SNe(D,\mathcal G_i)$\footnote{According to Theorem~\ref{prop:spurious}, an edge $D-A$ in the LEG, with $D$ in the $h$-hop neighborhood and $A$ outside, implies that $A$ is either a neighbor or a descendant of $D$ in $\mathcal{G}_i$.}.
    \item \textbf{Arrow edge:} $(D_1 \rightarrow D_2) \in \mathbb{E}^{Y,h}$ if and only if, $\forall \mathcal G_i\in LMEC(Y, h, \mathcal{G})$, $D_1 \rightarrow D_2$ appears in $\mathcal{G}_i$.
    \item \textbf{Double-bar edge:} $(D \arrowdoublebar A) \in \mathbb{E}^{Y,h}$ if and only if, $\forall \mathcal G_i \in LMEC(Y, h, \mathcal{G})$ and $\forall W \notin \{NeHood(Y, h, \mathcal{G}_i)\cup Ne(Y,\mathcal G_i)\cup SNe(Y,\mathcal G_i)\}$, $(D,A,W)$ is a non-collider triple in $\mathcal G_i$.
\end{enumerate}
\end{definition}

Note that from this definition, the absence of an edge between two nodes in the LEG has different interpretations depending on the nodes: an absent edge indicates non-adjacency in all DAGs of $LMEC(Y,h,\mathcal{G})$ if at least one node is in $NeHood(Y,h,\mathcal G)$, but conveys no information if both nodes lie outside the neighborhood.
The above definition may be relatively dense, and in general graphical concepts are often better understood through visual representation. 
To illustrate this, Figure~\ref{fig:dag_two_legs} presents: 
(a) the original DAG; 
(b) the LEG relative to the $0$-hop neighborhood of $Y$; 
(c) the LEG relative to the $1$-hop neighborhood; 
and (d) the LEG relative to the $2$-hop neighborhood. 
As hop $h$ increases, more edges become oriented, including among nodes that were already present in the previous neighborhoods.

\begin{figure}[t]
    \centering
        \begin{minipage}{.24\textwidth}
        \centering
    \begin{subfigure}{}{$\mathcal{G}$.}
        \centering
        
        \begin{tikzpicture}[{black, circle, draw, inner sep=0}]
            \tikzset{nodes={draw,rounded corners, minimum height=0.6cm, minimum width=0.6cm}}

            \node (X)  at (0,0) [fill=blue!30] {$X$};
            \node (M)  at (0,.9) {$M$};
            \node (Z2) at (0,-.9) {$Z_2$};
            \node (Z3) at (.9,-.9) {$Z_3$};
            \node (Z4) at (1.8,-.9) {$Z_4$};
            \node (Z5) at (0,-1.8) {$Z_5$};
            \node (Z6) at (.9,-1.8) {$Z_6$};
            \node (Z7) at (1.8,-1.8) {$Z_7$};
            \node (Y)  at (.9,0) [fill=red!30] {$Y$};
            \node (Z1) at (1.8,0) {$Z_1$};

            \draw[->, red, >=latex] (X) -- (Y);
            \draw[->, >=latex] (X) -- (M);
            \draw[->, >=latex] (M) -- (Y);
            \draw[->, >=latex] (Y) -- (Z1);
            \draw[->, >=latex] (Z2) -- (X);
            \draw[->, >=latex] (Z2) -- (Y);
            \draw[->, >=latex] (Z2) -- (Z3);
            \draw[->, >=latex] (Z3) -- (X);
            \draw[->, >=latex] (Z3) -- (Z1);
            \draw[->, >=latex] (Z4) -- (Z1);
            \draw[->, >=latex] (Z5) -- (Z3);
            \draw[->, >=latex] (Z6) -- (Z3);
            \draw[->, >=latex] (Z7) -- (Z4);
        \end{tikzpicture}
        \label{fig:dag_two_legs:dag}
    \end{subfigure}
        \end{minipage}
\hfill
        \begin{minipage}{.24\textwidth}
        \centering
        \begin{subfigure}{}{$\mathcal{L}^{Y,0}$.}
        \centering
        
        \begin{tikzpicture}[{black, circle, draw, inner sep=0}]
            \tikzset{nodes={draw,rounded corners, minimum height=0.6cm, minimum width=0.6cm}}

            \node (X)  at (0,0) [fill=blue!30] {$X$};
            \node (M)  at (0,.9) {$M$};
            \node (Z2) at (0,-.9) {$Z_2$};
            \node (Z3) at (.9,-.9) {$Z_3$};
            \node (Z4) at (1.8,-.9) {$Z_4$};
            \node (Z5) at (0,-1.8) {$Z_5$};
            \node (Z6) at (.9,-1.8) {$Z_6$};
            \node (Z7) at (1.8,-1.8) {$Z_7$};
            \node (Y)  at (.9,0) [fill=red!30] {$Y$};
            \node (Z1) at (1.8,0) {$Z_1$};

            \draw[-, color = red] (X) -- (Y);
            \draw[-] (M) -- (Y);
            \draw[-||_||,>=latex] (Y) -- (Z1);
            \draw[-] (Y) -- (Z2);
        \end{tikzpicture}
        \label{fig:dag_two_legs:leg1}
    \end{subfigure}
        \end{minipage}
\hfill
        \begin{minipage}{.24\textwidth}
        \centering
    \begin{subfigure}{}{$\mathcal{L}^{Y,1}$.}
        \centering
        
        \begin{tikzpicture}[{black, circle, draw, inner sep=0}]
            \tikzset{nodes={draw,rounded corners, minimum height=0.6cm, minimum width=0.6cm}}

            \node (X)  at (0,0) [fill=blue!30] {$X$};
            \node (M)  at (0,.9) [fill=gray!30] {$M$};
            \node (Z2) at (0,-.9) [fill=gray!30] {$Z_2$};
            \node (Z3) at (.9,-.9)  {$Z_3$};
            \node (Z4) at (1.8,-.9)  {$Z_4$};
            \node (Z5) at (0,-1.8)  {$Z_5$};
            \node (Z6) at (.9,-1.8)  {$Z_6$};
            \node (Z7) at (1.8,-1.8)  {$Z_7$};
            \node (Y)  at (.9,0) [fill=red!30] {$Y$};
            \node (Z1) at (1.8,0) [fill=gray!30] {$Z_1$};

            \draw[-, red, >=latex] (X) -- (Y);
            \draw[-, >=latex] (X) -- (M);
            \draw[->, >=latex] (M) -- (Y);
            \draw[->, >=latex] (Y) -- (Z1);
            \draw[-, >=latex] (Z2) -- (X);
            \draw[->, >=latex] (Z2) -- (Y);
            \draw[-||_||, >=latex] (Z2) -- (Z3);
            \draw[-, >=latex] (Z3) -- (X);
            \draw[-, >=latex] (Z3) -- (Z1);
            \draw[-, >=latex] (Z4) -- (Z1);
        \end{tikzpicture}
        \label{fig:dag_two_legs:leg2}
    \end{subfigure}
        \end{minipage}
    \hfill
    \begin{minipage}{.24\textwidth}
        \centering
    \begin{subfigure}{}{$\mathcal{L}^{Y,2}$.}
        \centering
        
        \begin{tikzpicture}[{black, circle, draw, inner sep=0}]
            \tikzset{nodes={draw,rounded corners, minimum height=0.6cm, minimum width=0.6cm}}

            \node (X)  at (0,0) [fill=blue!30] {$X$};
            \node (M)  at (0,.9) [fill=gray!30] {$M$};
            \node (Z2) at (0,-.9) [fill=gray!30] {$Z_2$};
            \node (Z3) at (.9,-.9) [fill=gray!30] {$Z_3$};
            \node (Z4) at (1.8,-.9) [fill=gray!30] {$Z_4$};
            \node (Z5) at (0,-1.8) [fill=gray!30] {$Z_5$};
            \node (Z6) at (.9,-1.8)  {$Z_6$};
            \node (Z7) at (1.8,-1.8)  {$Z_7$};
            \node (Y)  at (.9,0) [fill=red!30] {$Y$};
            \node (Z1) at (1.8,0) [fill=gray!30] {$Z_1$};

            \draw[->, red, >=latex] (X) -- (Y);
            \draw[->, >=latex] (X) -- (M);
            \draw[->, >=latex] (M) -- (Y);
            \draw[->, >=latex] (Y) -- (Z1);
            \draw[->, >=latex] (Z2) -- (X);
            \draw[->, >=latex] (Z2) -- (Y);
            \draw[->, >=latex] (Z2) -- (Z3);
            \draw[->, >=latex] (Z3) -- (X);
            \draw[->, >=latex] (Z3) -- (Z1);
            \draw[->, >=latex] (Z4) -- (Z1);
            \draw[->, >=latex] (Z5) -- (Z3);
            \draw[-, >=latex] (Z6) -- (Z3);
            \draw[-||_||, >=latex] (Z4) -- (Z7);
        \end{tikzpicture}
        \label{fig:dag_two_legs:leg3}
    \end{subfigure}
    \end{minipage}
   \caption{A DAG $\mathcal{G}$ and the LEGs $\mathcal{L}^{Y,0}$, $\mathcal{L}^{Y,1}$, and $\mathcal{L}^{Y,2}$ around node $Y$. Red: outcome (target); blue: treatment; grey: $h$-neighborhood nodes; red arrow: direct effect.}
    \label{fig:dag_two_legs}
\end{figure}

A LEG is particularly valuable in the context of local causal discovery, as it compactly encodes rich information about the local $d$-separations within the graph.
Moreover, given any DAG from a specific LMEC, it is possible to reconstruct the corresponding LEG, as demonstrated in the following theorem.

\begin{restatable}{theorem}{mytheoremtwo}
\label{theorem:leg}
Let $\mathcal{G} = (\mathbb{V}, \mathbb{E})$ be a DAG, $Y \in \mathbb{V}$ a target, and $h \ge 0$ a hop.  
The LEG $\mathcal{L}^{Y,h}$ associated with $LMEC(Y,h,\mathcal{G})$ can be constructed as follows:
\begin{enumerate}
    \item \textbf{Neighborhood:}  
    $NeHood(Y, h, \mathcal{L}^{Y,h}) = NeHood(Y,h,\mathcal G).$

    \item \textbf{Local skeleton:}  
    $\forall D_1,D_2 \in NeHood(Y,h,\mathcal G), \; D_1\in Ne(D_2, \mathcal{G}) \Rightarrow D_1\in Ne(D_2, \mathcal{L}^{Y,h}) $

    \item \textbf{Local UC:}  
    $\forall D_1, D_2, D_3 \in NeHood(Y,h,\mathcal G)$ such that $(D_1, D_2, D_3)$ forms a UC in $\mathcal{G}$, the same UC is preserved in $\mathcal{L}^{Y,h}.$

    \item \textbf{Outside neighbors:}  
    $\forall D \in NeHood(Y,h,\mathcal G), \; \forall A \notin NeHood(Y,h,\mathcal G), \; A \in Ne(D, \mathcal{L}^{Y,h})$  
    if and only if $A \in Ne(D,\mathcal G)\cup SNe(D, \mathcal{G}).$

    \item \textbf{Meek rules:}  
    Apply Meek rules iteratively to nodes in the $h$-neighborhood until no further rule can be applied.

    \item \textbf{"No non-collider" rule (NNC rule):}  
    $\forall D\in NeHood(Y,h,\mathcal G), A\notin NeHood(Y,h,\mathcal G)$ such that $A\in Ne(D,\mathcal L^{Y,h})$. If $\nexists W\notin \{NeHood(Y,h,\mathcal G)\cup Ne(D,\mathcal G)\cup SNe(D,\mathcal G)\}$ such that the triple $(D,A,W)$ forms a non-collider triple in $\mathcal G$, then $D\arrowdoublebar A\in \mathbb E^{Y,h}$. 
\end{enumerate}
\end{restatable}

Theorem~\ref{theorem:leg} along with Theorem~\ref{theorem:lmec} determine the graphical information accessible through local $d$-separations using a PC-style approach. The skeleton of the LEG restricted to the neighborhood coincides with that of the true DAG, with no spurious internal edges (items~1 and~2). Between the neighborhood and the rest of the graph, all true edges are preserved, but spurious edges may appear, corresponding to descendant inducing neighbors (items~4 and Theorem~\ref{prop:spurious}). Items~5 and~6 ensure that the edge orientations in the LEG are consistent with the DAG, and applying Meek’s rules allows further edge orientations as the neighborhood expands. 
We refer again to the DAG $\mathcal G$ and its LEGs $\mathcal L^{Y,0},\mathcal L^{Y,1}$ and $\mathcal L^{Y,2}$ in Figure~\ref{fig:dag_two_legs} to illustrate each item of Theorem~\ref{theorem:leg}. To illustrate items 1 to 3, remark that the skeleton and UCs among the neighborhood in the LEGs are the same as in $\mathcal G$. For example, the UC $M \to Y \leftarrow Z_2$ is in $\mathcal L^{Y,1}$ as in $\mathcal G$ since $M,Y$ and $Z_2$ belong to the neighborhood of hop 1 of $Y$. For item 4, remark that all adjacencies between the $h$-neighborhood nodes and their neighbors outside are in the LEGs, \eg $Z_1-Z_4$ in $\mathcal L^{Y,1}$. Item 5 is illustrated in $\mathcal L^{Y,2}$, where $Z_3 \to X$ due to $Z_5 \to Z_3$ (meek rule 1), or $Z_2 \to Z_3$ to prevent cycles (meek rule 2). Finally, the NNC rule of item 6 is illustrated in $\mathcal L^{Y,1}$ where the edge $Z_2 \arrowdoublebar Z_3$ appears because no non-collider triple involves $Z_2, Z_3$ and any node from $\{Z_4,Z_5,Z_6,Z_7\}$. By contrast, the edge $Z_1-Z_4$ is not oriented in this LEG due to the non-collider path $Z_1 \leftarrow Z_4 \leftarrow Z_7$.
For simplicity, $\mathcal G$ does not include structures that may induce spurious neighbors; see Appendix~\ref{appendix:spurious} for an example with spurious neighbors.

For a given LMEC, the corresponding LEG is unique, as emphasized by this corollary.

\begin{restatable}{corollary}{mycorollaryone}
Let $\mathcal{G}_1$ and $\mathcal{G}_2$ be two DAGs, and $\mathcal{L}_1^{Y,h}$ and $\mathcal{L}_2^{Y,h}$ their associated LEGs. If $LMEC(Y,h,\mathcal{G}_1) = LMEC(Y,h,\mathcal{G}_2)$, then $\mathcal{L}_1^{Y,h} = \mathcal{L}_2^{Y,h}$.
\end{restatable}

\section{Local causal discovery}
\label{sec:LCD}

This section presents \textbf{LocPC}, a local adaptation of the PC algorithm for recovering the LEG from an observed distribution via CI tests.
Given a target variable $Y$ and an integer $h$ specifying the size of the local neighborhood, \textbf{LocPC} starts with an undirected graph $\widehat{\mathcal{L}}$ where all nodes are connected to $Y$ and initializes the set of focus variables as $\mathbb{D} = \{Y\}$. \textbf{LocPC} then performs CI tests following the PC algorithm strategy. For each $D \in \mathbb{D}$ and $W \in Ne(D, \widehat{\mathcal{L}})$, it checks $(D \indep W \mid \mathbb{S})_\mathcal{P}$ for subsets $\mathbb{S} \subseteq \text{Adj}(D, \widehat{\mathcal{L}})$, starting with $|\mathbb{S}| = 0$ and incrementally increasing the size. If any subset $\mathbb{S}$ satisfies $(D \indep W \mid \mathbb{S})_\mathcal{P}$, the edge $D-W$ is removed and $\mathbb{S}$ is cached. This continues until either the edge is removed or all conditioning sets are exhausted. The first phase identifies variables that cannot be rendered conditionally independent of $Y$, including all true neighbors and potential spurious neighbors. In the second phase, these candidate neighbors are added to $\mathbb{D}$, undirected edges are created between them and other nodes, and the first phase is repeated while avoiding redundant tests. It is important to note that the algorithm performs a \emph{retest} whenever a CI is found between a node $D$ and previously visited nodes whose edge has not yet been pruned. This mechanism ensures that spurious neighbors within the neighborhood are pruned and confined to its periphery. The algorithm iteratively expands $\mathbb{D}$ by including neighbors of nodes in the current set, repeating this $h$ times, until $\mathbb{D}$ contains nodes at distance $h$ from $Y$. It then proceeds to edge orientation: first, all UCs in the $h$-neighborhood are detected as in the PC algorithm; next, Meek rules are applied iteratively in the $h$-neighborhood. Finally, for pairs $(D,A)$, edges $D \arrowdoublebar A$ are added to the LEG according to Definition~\ref{def:local_rules}.
This definition introduces a condition analogous to the NNC rule in item~6 of Theorem~\ref{theorem:leg}. However, since item~6 of Theorem~\ref{theorem:leg} assumes access to a DAG, 
we instead formulate a version based on CIs, accessible to the algorithm.

\begin{definition}[CI-based NNC rule] \label{def:local_rules}  
Let $D \in NeHood(Y,h,\mathcal G)$ and $A \notin NeHood(Y,h,\mathcal G)$ with $D-A$.  
If there is no $W \notin NeHood(Y,h,\mathcal G) \cup Ne(D,\mathcal{G}) \cup SNe(D,\mathcal{G})$ such that for all $\mathbb{S} \subseteq Ne(D,\mathcal{G})\cup SNe(D,\mathcal G)$ with $(D \nindep W \mid \mathbb{S})_{\mathcal{P}}$ we have $A \in \mathbb{S}$, then $D \arrowdoublebar A$.

\end{definition}

A pseudo-code of the \textbf{LocPC} algorithm and additional details on incorporating expert background knowledge (\eg, specifying in the algorithm sex as a non-descendant of post-birth variables) are provided in Appendix~\ref{appendix:pseudocode}. 


Discovering essential graphs typically requires both \textbf{causal sufficiency} and the faithfulness assumptions.
In the context of LEG recovery using \textbf{LocPC},  
a key insight is that the full faithfulness assumption is not required. Instead, we introduce a weaker assumption, which we refer to as \textbf{local faithfulness} which can be more realistic and practical  in many applications. A distribution $\mathcal{P}$ satisfies the \textbf{local faithfulness assumption} with respect to a DAG $\mathcal{G}$, a hop $h$ and a target $Y$ if, for every 
$D \in NeHood(Y,h,\mathcal{G})$, every $s \ge 1$, every 
$A \in \mathbb{C}_{s-1}(D,\mathcal{G})$, and every 
$\mathbb{Z} \subseteq \mathbb{C}_{s-1}(D,\mathcal{G})$ such that $|\mathbb{Z}| = s$, we have  
$ (D \indep A \mid \mathbb{Z})_{\mathcal{P}}\Rightarrow (D \indep A \mid \mathbb{Z})_{\mathcal{G}}$. This assumption is weaker than full faithfulness, as it does not, for instance, impose faithfulness between nodes outside the $h$-neighborhood. Under these necessary assumptions now established, \textbf{LocPC} algorithm is correct.






\begin{restatable}[Correctness of LocPC]{theorem}{mytheoremthree}
\label{theorem:faithfulness_implications}
Let $\widehat{\mathcal{L}}$ be the \textbf{LocPC} output, and $\mathcal{L}^{Y,h}$ be the true LEG.  
Under \textbf{causal sufficiency} and \textbf{local faithfulness}, with perfect CI, $\widehat{\mathcal{L}}=\mathcal{L}^{Y,h}$.
\end{restatable}

Finally, the theoretical complexity of \textbf{LocPC} is given by the following proposition.
\begin{restatable}[Complexity of LocPC]{proposition}{mytheoremcomplexity}
\label{theorem:complexity}
Let \(n := |\mathbb V|\) and \(k_\ell := k_d + k_i\) with \(k_d := \max\{|Ne(D,\mathcal G)| : D \in NeHood(Y,h,\mathcal G)\}\) and \(k_i := \max\{|DINe(D,\mathcal G)| : D \in NeHood(Y,h,\mathcal G)\}\). The number of CI tests performed by \textbf{LocPC} to discover $\mathcal L^{Y,h}$ is bounded by
$\left(1 + k_\ell \frac{1 - k_d^h}{1 - k_d}\right) (n-1) \sum_{s=0}^{k_\ell} \binom{n-2}{s}\mathop{=}\limits_{n}\mathcal O(n^{k_\ell+1}),$
where $\mathop{=}\limits_{n}$ denotes the $\mathcal O$-complexity as $n\to +\infty$ while keeping $k_d,k_i,k_\ell$ fixed.
\end{restatable}

The complexity of \textbf{LocPC} as the graph size $n$ increases resembles that of PC, $\mathcal{O}(n^{d+2})$~\citep{Claassen_2013} for fixed $d:=\max\{|Ne(D,\mathcal G)|:D\in \mathbb V\}$, 
but direct comparison is subtle: \textbf{LocPC} is asymptotically better whenever $k_\ell + 1 < d$.
Two cases illustrates a theoretical local advantage to \textbf{LocPC}:  
(i) no descendant-inducing paths ($k_i = 0$);  
(ii) high-degree nodes outside the $h$-neighborhood ($d \gg k_d$) even if $k_i > 0$. These are theoretical worst-case results; in practice, \textbf{LocPC} may rarely reach this bound and will outperforms PC empirically (see for instance the experimental results in Section~\ref{sec:exp}).

\section{Local causal discovery for identifying CDE}
\label{sec:LCD-CDE}

In this section, we aim to recover a portion of the LEG sufficient to determine the identifiability of $CDE(x, x', Y)$. According to \cite[Theorem 5.4]{Flanagan_2020}, identifying this CDE requires verifying whether all edges adjacent to $Y$ are oriented. A naive strategy, consists of initially applying the \textbf{LocPC} algorithm with $h=1$, and subsequently checking whether all edges incident to $Y$ have been oriented. If so, the $CDE(x, x', Y)$ is identifiable, and the procedure terminates. Otherwise, the process is repeated with $h=2$, reusing previously obtained information and avoiding redundant CI tests, and continues incrementally in this manner.
%
For instance, consider the CDE of $X \to Y$ (depicted as a red edge) in Figure~\ref{fig:dag_two_legs}. This effect is identifiable in $\mathcal{L}^{Y,2}$ because all edges adjacent to $Y$ are oriented in the  LEG.
When $CDE(x,x',Y)$ is not identifiable from the full essential graph, this naive approach would repeatedly apply \textbf{LocPC} until the entire graph is recovered. However, one can anticipate non-identifiable cases and determine variables whose exploration would be uninformative (i.e., adding them to $\mathbb{D}$). To this end, we introduce a  criterion that detects in advance when an edge into $Y$ is non-orientable along with a theorem that establishes its soundness.

\begin{definition}[Non-orientability criterion, NOC]
\label{def:orientability}
Let $\mathbb D \subseteq NeHood(Y,h,\mathcal G)$, and $\mathcal{L}^{Y,h}=(\mathbb V, \mathbb E^{Y,h})$ with $h \ge 1$. $\mathbb D$ satisfies the non-orientability criterion (NOC) if, for all $D \in \mathbb D$, there is no $A \notin \mathbb D$ such that $(D-A)\in \mathbb E^{Y,h}$, and at most one $A \notin \mathbb D$ satisfies $(D\arrowdoublebar A)\in \mathbb E^{Y,h}$.
\end{definition}
\begin{restatable}[NOC implies non-identifiability]{theorem}{mycorollarytwo}
\label{corollary:identifiability_CDE}
If $\exists \mathbb D \subseteq NeHood(Y, h, \mathcal G)$ with $Y \in \mathbb D$ satisfying NOC in $\mathcal{L}^{Y,h}$, then $CDE(x,x',Y)$ is not identifiable.
\end{restatable}

\begin{figure}[t!]
    \centering
        \begin{minipage}{.32\textwidth}
    \centering
        \begin{subfigure}{}{$\mathcal G$.}
    \centering
    
    \begin{tikzpicture}[{black, circle, draw, inner sep=0}]
			\tikzset{nodes={draw,rounded corners},minimum height=0.6cm,minimum width=0.6cm}	
            \node (Y)  at (-.9,0) [fill = red!30] {$Y$};
            \node (X) at (-.9,.9) [fill = blue!30] {$X$};
            \node (A1) at (0,.9)  {$A_1$};
            \node (D1) at (-.9,-.9) {$D_1$};
            \node (D2)  at (0,0) {$D_2$};
            \node (A2)  at (0,-.9) {$A_2$};
            \node (W2) at (.9,-.9) {$W_2$};
            \node (W1)  at (.9,.9) {$W_1$};
            \node (Z)  at (.9,0) {$Z$};
            \draw[->,>=latex, color = red] (X) -- (Y);
            \draw[->, >=latex, bend right] (X) to (D1);
            \draw[->,>=latex] (D1) -- (Y);
            \draw[->,>=latex] (X) -- (A1);
            \draw[->,>=latex] (Y) -- (D2);
            \draw[->,>=latex] (D2) -- (A2);
            \draw[->,>=latex] (W1) -- (A1);
            \draw[->,>=latex] (A2) -- (W2);
            \draw[->,>=latex] (W1) -- (Z);
            \draw[->,>=latex] (Z) -- (W2);
            \draw[->,>=latex] (A1) -- (D2);
        \end{tikzpicture}
        \label{fig:non-orientability:EG}
    \end{subfigure}
        \end{minipage}
    \hfill
    \begin{minipage}{.32\textwidth}
        \centering
    \begin{subfigure}{}{$\mathcal C$.}
    \centering
    
    \begin{tikzpicture}[{black, circle, draw, inner sep=0}]
			\tikzset{nodes={draw,rounded corners},minimum height=0.6cm,minimum width=0.6cm}	
            \node (Y)  at (-.9,0) [fill = red!30] {$Y$};
            \node (X) at (-.9,.9) [fill = blue!30] {$X$};
            \node (A1) at (0,.9)  {$A_1$};
            \node (D1) at (-.9,-.9) {$D_1$};
            \node (D2)  at (0,0) {$D_2$};
            \node (A2)  at (0,-.9) {$A_2$};
            \node (W2) at (.9,-.9) {$W_2$};
            \node (W1)  at (.9,.9) {$W_1$};
            \node (Z)  at (.9,0) {$Z$};
            \draw[-,>=latex, color = red] (X) -- (Y);
            \draw[-, >=latex, bend right] (X) to (D1);
            \draw[-,>=latex] (D1) -- (Y);
            \draw[->,>=latex] (X) -- (A1);
            \draw[->,>=latex] (Y) -- (D2);
            \draw[->,>=latex] (A1) -- (D2);
            \draw[->,>=latex] (D2) -- (A2);
            \draw[->,>=latex] (W1) -- (A1);
            \draw[->,>=latex] (A2) -- (W2);
            \draw[-,>=latex] (W1) -- (Z);
            \draw[->,>=latex] (Z) -- (W2);
        \end{tikzpicture}
        \label{fig:non-orientability:EG}
    \end{subfigure}
        \end{minipage}
    \hfill 
        \begin{minipage}{.32\textwidth}
            \centering
    \begin{subfigure}{}{$\mathcal L^{Y,1}$.}
    \centering
    
    \begin{tikzpicture}[{black, circle, draw, inner sep=0}]
			\tikzset{nodes={draw,rounded corners},minimum height=0.6cm,minimum width=0.6cm}	
            \node (Y)  at (-.9,0) [fill = red!30] {$Y$};
            \node (X) at (-.9,.9) [fill = blue!30] {$X$};
            \node (A1) at (0,.9)  {$A_1$};
            \node (D1) at (-.9,-.9) [fill = gray!30] {$D_1$};
            \node (D2)  at (0,0) [fill = gray!30] {$D_2$};
            \node (A2)  at (0,-.9) {$A_2$};
            \node (W2) at (.9,-.9) {$W_2$};
            \node (W1)  at (.9,.9) {$W_1$};
            \node (Z)  at (.9,0) {$Z$};
            \draw[-,>=latex, color = red] (X) -- (Y);
            \draw[->,>=latex] (A1) -- (D2);
            \draw[-, >=latex, bend right] (X) to (D1);
            \draw[-,>=latex] (D1) -- (Y);
            \draw[-||_||,>=latex] (X) -- (A1);
            \draw[->,>=latex] (Y) -- (D2);
            \draw[-,>=latex] (D2) -- (A2);
        \end{tikzpicture}
        \label{fig:non-orientability:LEG}
    \end{subfigure}
    \end{minipage}
    \caption{DAG $\mathcal G$, essential graph $\mathcal C$, and LEG $\mathcal L^{Y,1}$. $\mathbb D = \{Y, X, D_1\}$ satisfies the NOC (Def.~\ref{def:orientability}), so Theorem~\ref{corollary:identifiability_CDE} implies that CDE is not identifiable, even with global discovery.}
\label{fig:corollary}
\end{figure}

Theorem~\ref{corollary:identifiability_CDE} shows that full recovery of the essential graph is unnecessary when the CDE is non-identifiable; the algorithm can terminate early if NOC holds. 
For illustration, consider the DAG $\mathcal G$, its essential graph $\mathcal{C}$, and its  LEG $\mathcal{L}^{Y,1}$ in Figure~\ref{fig:corollary} where the $CDE(x,x',Y)$ is not identifiable from $\mathcal{C}$. In $\mathcal{L}^{Y,1}$, consider the set $\mathbb{D} = \{Y, X, D_1\}$. There is no unoriented edge between any node in $\mathbb{D}$ and any node outside $\mathbb{D}$ (satisfying condition~1 of Definition~\ref{def:orientability}), and there is a unique $\arrowdoublebar$ edge $X \arrowdoublebar A_1$ (satisfying condition~2). Since both conditions hold, $\mathbb{D}$ satisfies NOC. Using Theorem~\ref{corollary:identifiability_CDE}, we can directly conclude from $\mathcal{L}^{Y,1}$ that $CDE(x,x',Y)$ is not identifiable from $\mathcal{C}$, without having to discover $\mathcal C$. 
In contrast, the set $\mathbb{D}' = \{Y, X, D_1, D_2\}$ does not satisfy NOC, as $D_2 - A_2$ violates condition~1.

Building on this idea, we propose the \textbf{LocPC-CDE} algorithm.  
Starting from the target variable $Y$, the algorithm incrementally constructs the LEG by expanding the neighborhood hop $h$.  
At each step, \textbf{LocPC} performs local causal discovery, using information from previous iterations to avoid redundant tests. After each expansion, the algorithm checks whether a set satisfies the NOC. If so, causal discovery can stop early, as $CDE(x,x',Y)$ is already known to be non-identifiable; otherwise, the process continues. Additional stopping criteria, such as detecting that $X$ is a child of $Y$ or that $X$ is non-adjacent to $Y$ (implying $CDE(x,x',Y) = 0$), are also included. A pseudocode version is provided in Appendix~\ref{appendix:pseudocode}. 
The following theorem establishes the correctness of \textbf{LocPC-CDE}.

\begin{restatable}{theorem}{mytheoremsix}
\label{th:sound_and_complete-LocPC-CDE}
    If \textbf{causal sufficiency} and \textbf{local faithfulness} are satisfied and with access to perfect CI, the \textbf{LocPC-CDE} algorithm will correctly detect if $CDE(x,x',Y)$ is identifiable and in case of identifiability it will return the valid adjustment set $Pa(Y,\mathcal G)$.
\end{restatable}

Obviously, \textbf{LocPC-CDE} is more efficient than the naive approach, as it can terminate early when the CDE is non-identifiable.
Moreover, the following result emphasizes that \textbf{LocPC-CDE} will identify $CDE(x,x',Y)$ as fast as possible using iteratively \textbf{LocPC}. 

\begin{restatable}{theorem}{mypropositiontwo}
\label{corollary:optimality_LocPC-CDE}
    Consider \textbf{causal sufficiency} and \textbf{local faithfulness} are satisfied and we have access to perfect CI, if \textbf{LocPC-CDE} returns that $CDE(x,x',Y)$ is not identifiable, then it was impossible to determine this at any earlier iteration of \textbf{LocPC-CDE}.
\end{restatable}

\section{Experiments}\label{sec:exp}
In this section, we empirically validate our theoretical results on simulated data as well as on real data.
Details about source code, data-generative processes, and data availability related to the experiments are deferred to Appendix~\ref{appendix:experiments}. 

\subsection{Synthetic data}

This section evaluates the \textbf{LocPC-CDE} algorithm by comparing it to the PC algorithm~\citep{Spirtes_2000}, to CMB~\citep{Gao_2015} and MBbyMB~\citep{wang2014discovering} designed to discovered every direct causes and effects around the target node $(Y)$, and to LDECC~\citep{Gupta_2023}, initially designed for local causal discovery of the total effect when targeting the treatment node but usable to identify a CDE when targeting the outcome node.

\begin{figure}[t!]
    \centering
    \begin{minipage}{.45\textwidth}
    \begin{subfigure}{}{Identifiable CDEs.}
        \centering
        
    \includegraphics[width=\linewidth]{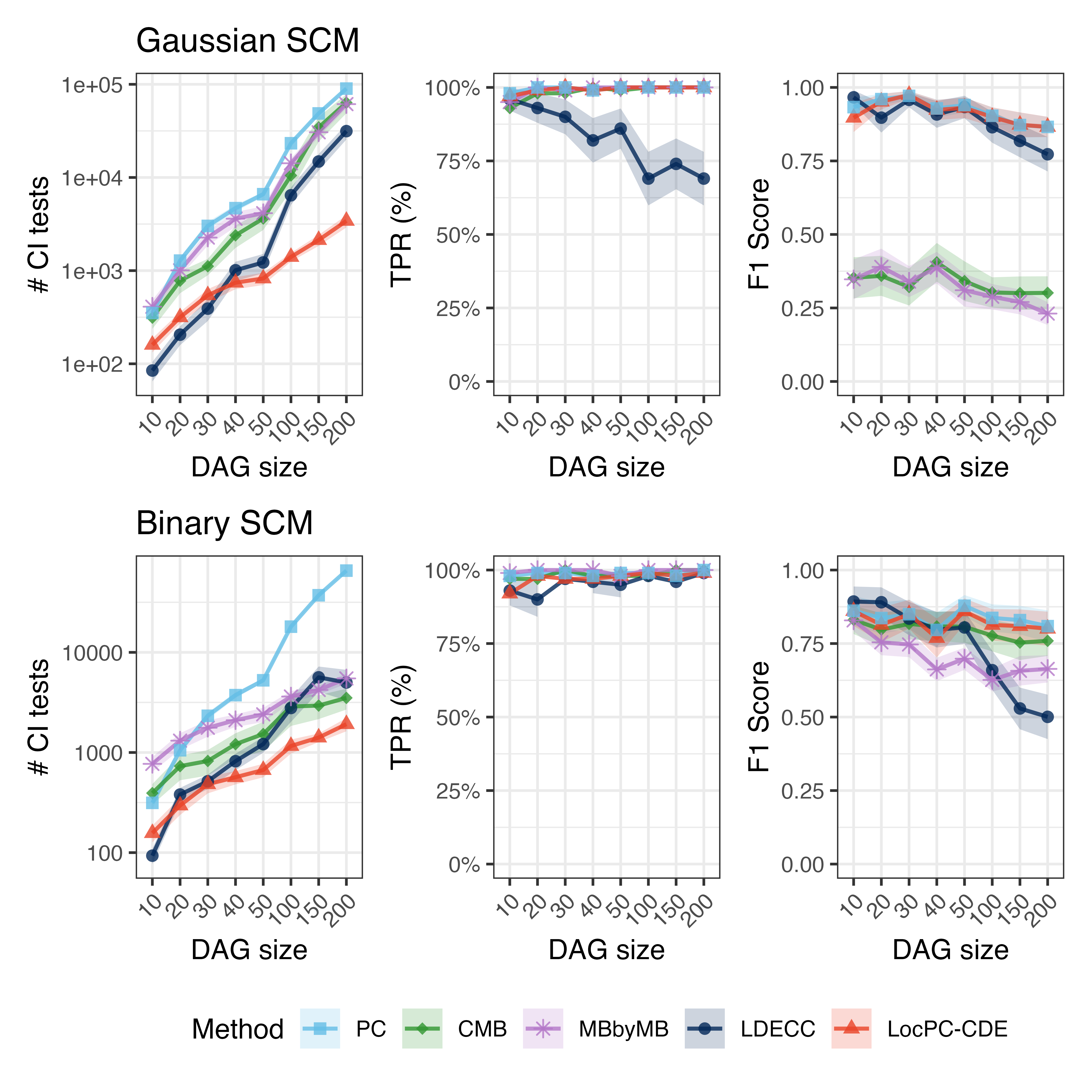}
        \label{fig:identifiable}
    \end{subfigure}
    \end{minipage}
    \hfill
    \begin{minipage}{.45\textwidth}
    \begin{subfigure}{}{Non-identifiable CDEs.}
        \centering
        
    \includegraphics[width=\linewidth]{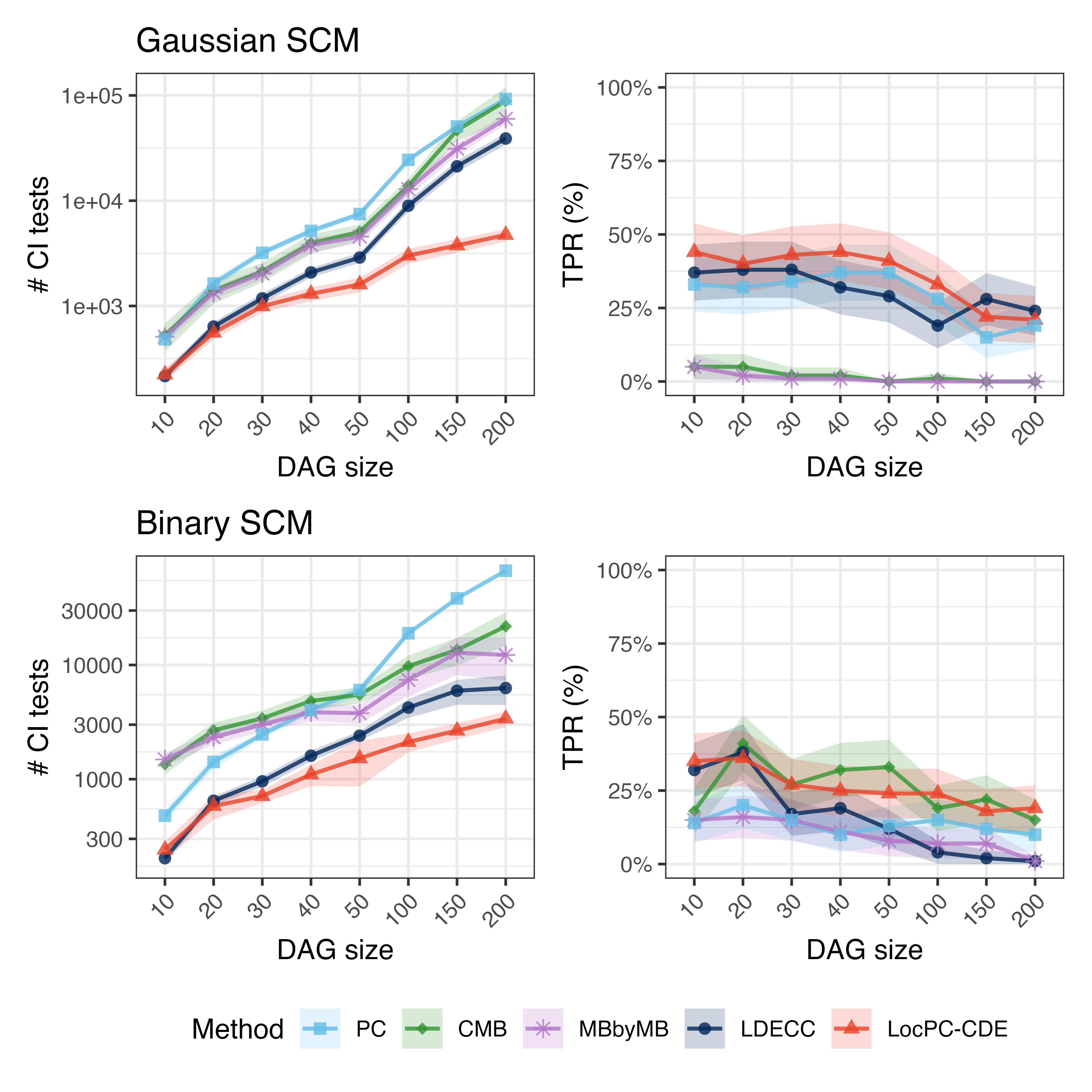}
        \label{fig:non_identifiable}
    \end{subfigure}
    \end{minipage}
    \caption{Empirical performance of \textbf{LocPC-CDE} across different graph sizes and SCM settings, compared to global discovery (PC) and state-of-the-art local discovery methods.}
    \label{fig:exp}
\end{figure}

We perform an evaluation on synthetic data, considering two settings: identifiable and non-identifiable CDEs based on independence tests.  
Algorithm performance is assessed for DAGs of size $|\mathbb V| \in [\![ 10, 200 ]\!]$. For each setting, 100 random Erdos-Rényi DAGs are simulated such that the CDE is either identifiable or not in the ground truth, and a sample of $n = 5000$ observations is generated from a linear Gaussian or binary SCM compatible with the DAG. Each algorithm is then applied using a Fisher-z CI test~\citep{fisher-1921} for continuous variables and a $G^2$ test~\citep{tsamardinos-2006} for binary variables, with standard significance level $\alpha = 0.05$. For each experiment, we measure: (1) the number of CI tests, (2) the proportion of DAGs where the CDE is correctly detected as identifiable or non-identifiable (true positive rate, TPR), and (3) in the identifiable case, the $F_1$ score between the estimated parents of the outcome $Y$ and the true parents, corresponding to a valid adjustment set for the CDE ($F_1=1$ indicates a perfect adjustment set). Results (mean $\pm 1.96\times$sd) are shown in Figure~\ref{fig:exp}. Regarding the number of CI tests, \textbf{LocPC-CDE} almost always achieves the best performance compared to other baselines, and systematically as DAG size increases. LDECC is only occasionally slightly better in small linear DAG settings, while PC unsurprisingly remains the baseline performing the most tests. Interestingly, Markov Blanket approaches (CMB and MBbyMB) sometimes perform more tests than PC in binary settings. For identifiable DAGs, the TPR is close to 1 for all baselines and settings, except for LDECC, which decreases in linear DAGs as size grows, indicating difficulty in orienting edges around the outcome in this setting. Regarding $F_1$ scores, Markov Blanket methods perform worse in linear settings, while \textbf{LocPC-CDE}, LDECC, and PC show similar performance, with LDECC degrading on large graphs. In non-identifiable DAGs, TPRs are relatively low (below 50\%) for all baselines, demonstrating a tendency to over-orient or over-prune edges, with relatively better performance for \textbf{LocPC-CDE}, PC, and LDECC in linear graphs, and for CMB and \textbf{LocPC-CDE} in binary graphs. In conclusion, this experiment shows the overall outperformance of \textbf{LocPC-CDE} in terms of CI tests as DAG size increases, while remaining competitive with other baselines across metrics, emphasizing its suitability for local causal discovery.


\subsection{Real-data}

We performed causal discovery on aggregated data from the French National Health Data System (SNDS), covering the incidence of various pathologies across 101 departments in 2023.
All baselines were run to identify the parents of diabetes using a $z$-Fisher correlation test. LDECC and MBbyMB found no parents after 56 and 1310 CI tests, respectively; CMB identified maternity in 1634 tests; while PC and \textbf{LocPC-CDE} both recovered chronic kidney disease and vascular pathologies in 1613 and 170 tests, respectively. Thus, \textbf{LocPC-CDE} matches PC while performing around 10 times fewer tests. The variables identified by PC and LocPC-CDE align with known associations~\citep{kumar2023bidirectional,dekamin2022cvd_t2d}; however, due to the limited sample size and potential model assumption violations, formal causal relationships cannot be established from these aggregated data.


\section{Conclusion}
\label{sec:conc}


In this work, we addressed PC-style local causal discovery for identifying and orienting the portion of a DAG necessary and sufficient to estimate a given CDE. Our main contributions are the characterization of potential spurious edges in PC-style local discovery, the formalization of the LMEC, and the introduction of the LEG to represent the information recoverable from local $d$-separations. Under standard causal sufficiency and a \textit{local} faithfulness assumptions, we propose \textbf{LocPC} to recover the LEG and 
\textbf{LocPC-CDE} for CDE estimation, leveraging the LMEC characterization to define a stopping criterion. We prove correctness and demonstrate efficiency of \textbf{LocPC-CDE} on synthetic and real data.

We emphasize that when estimating the CDE using the adjustment set detected by \textbf{LocPC-CDE}, caution is needed to avoid \textit{double dipping}, i.e., using the same data for discovery and estimation, which can bias confidence intervals~\citep{Gradu03042025}. A simple remedy is to split the data into disjoint discovery and estimation sets. Alternatively, repeating causal discovery with noise injection in the test statistics can mitigate this bias, 
but requires asymptotic assumptions on the CI test statistics~\citep{chang2025postselection}.
Moreover, while this work focuses on CDE identification, the findings of this paper can be applied to other targets requiring edge orientations in the neighborhood of a target. For instance, for the total effect, orienting edges around the treatment allows identification via its parent set as a valid adjustment set. However, in such cases, the stopping criterion and completeness of \textbf{LocPC-CDE} are not guaranteed: an undirected edge connected to the treatment does not necessarily prevent identifiability~\citep{Perkovic_2018}. 

This work has two main limitations: the LEG characterization is incomplete (a  LEG may correspond to different LMECs) and both \textbf{LocPC} and \textbf{LocPC-CDE} assume causal sufficiency. Relaxing the causal sufficiency assumption for local constraint-based causal discovery, in the spirit of the FCI algorithm and partial ancestral graph representations~\citep{Spirtes_2000,Zhang_2008}, is a promising direction for future work. 



\subsection*{Acknowledgement}
We thank Séhane Bel-Houari Durand for several discussion about the complexity of the algorithm, and Simon Ferreira for insightful discussions about descendant inducing paths. This work was supported by the CIPHOD project (ANR-23-CPJ1-0212-01) and by funding from the French government managed by the National Research Agency (ANR) under the France 2030 program (ANR-23-IACL-0007).

\bibliographystyle{plainnat}
\bibliography{references.bib}

\newpage
\appendix

\section{Proofs}
\label{appendix:proofs}

\subsection{Proof of Theorem~\ref{prop:spurious}}

We introduce and prove a preliminary lemma.

\begin{lemma}
\label{lemma:DINe}
Let $\mathcal{G} = (\mathbb{V}, \mathbb{E})$ be a DAG, $A \in \mathbb{V}$ and $B \notin Ne(A,\mathcal{G})$. $(A \nindep B \mid \mathbb{Z})_\mathcal{G}$ for all $\mathbb{Z} \subseteq Ne(A,\mathcal{G})$ if and only if $B\in DINe(A,\mathcal G)$.
\end{lemma}

\begin{proof}
We first prove $(\Leftarrow)$. Suppose $B\in DINe(A,\mathcal G)$. Then there exists a DIP $\pi$ between $A$ and $B$ relative to $\mathbb L=(L_1=A,L_2,\cdots,L_{k-1},B=L_k)$. By definition of a DIP, there exist a directed path $A\to \cdots \to L_{k-1}\leftarrow \cdots\to B$, and $L_{k-1}\in Ne(A,\mathcal G)$. Hence there exists a path $\widetilde\pi$ of the form $A\to L_{k-1}\leftarrow\cdots\to B$ which is also a DIP so whose unique collider is $L_{k-1}$. Since $\widetilde \pi$ is a DIP, there exists a chain $A\to L_{k-1}\to \cdots \to B$ which requires conditioning on $L_{k-1}$ to be blocked by a node in $Ne(A,\mathcal G)$. But conditioning on $L_{k-1}$ activates the collider of $\widetilde\pi$ and $L_{k-1}$ is the only neighbor of $A$ on $\widetilde\pi$, so no other node in $\mathbb Z\subseteq Ne(A,\mathcal G)$ can block this path. Therefore $\widetilde\pi$ remains active for every $\mathbb Z\subseteq Ne(A,\mathcal G)$, implying $(A\nindep B\mid\mathbb Z)_\mathcal G$ for all such $\mathbb Z$.

We now prove $(\Rightarrow)$. Let $B\notin Ne(A,\mathcal G)$ be such that $(A \nindep B \mid \mathbb Z)_\mathcal G$ for all $\mathbb Z \subseteq Ne(A,\mathcal G)$. First, there must exist an active path between $A$ and $B$, otherwise $(A\indep B \mid \emptyset)_\mathcal G$, which contradicts the assumption since $\emptyset \subseteq Ne(A,\mathcal G)$. Second, $B\in De(A,\mathcal G)$. Indeed, if $B\notin De(A,\mathcal G)$, then by $d$-separation, $(A\indep B \mid Pa(A,\mathcal G))_\mathcal G$. Hence, any potential active path between $A$ and $B$ is either a directed path $A\to \cdots \to B$ or a backdoor path $A\leftarrow \cdots \to B$. All backdoor paths can be blocked by conditioning on the neighbor of $A$, which cannot activate any other path between $A$ and $B$ since it cannot be a collider. Therefore, the only paths preventing $d$-separation of $A$ and $B$ by a subset of $Ne(A,\mathcal G)$ are directed paths of the form $A\to \cdots \to B$ which cannot be blocked without activating another path. Indeed, such directed paths can always be blocked by conditioning on any neighbor of $A$ along the path. However, if $A$ and $B$ cannot be $d$-separated by any subset of $Ne(A,\mathcal G)$, there exists at least one directed path, denoted $\pi_C$, such that for every $L_i\in Ne(A,\mathcal G)\cap V(\pi_C)$, conditioning on $L_i$ activates another path $\widetilde\pi$, i.e., $L_i$ is a collider on $\widetilde\pi$ and $\widetilde\pi$ is not blocked naturally by any other collider. By definition, this path $\widetilde\pi$ is then a DIP, which implies $B\in DINe(A,\mathcal G)$.
\end{proof}

We restate Theorem~\ref{prop:spurious} and then we prove it.

\mypropositionone*

\begin{proof}
    Firstly, remark that for any $N \in Ne(D,\mathcal G)$, there exists no set $\mathbb Z \subseteq \mathbb V \setminus \{D,N\}$ such that $(D \indep N \mid \mathbb Z)_\mathcal G$; hence, by definition of $\mathbb C_s$, we have $N \in \mathbb C_s$ for all $s \ge 0$. In other words, all neighbors of $D$ are always contained in $\mathbb C_s$. Then, for any $A \notin Ne(D,\mathcal G)$, if $A \in SNe(D,\mathcal G)$, then for $s = |\mathbb V|-2$, there exists no $\mathbb Z \subseteq \mathbb C_{s-1}$ such that $(D \indep A \mid \mathbb Z)_\mathcal G$; and in particular, by the previous remark, no $\mathbb Z \subseteq Ne(D,\mathcal G)$. By Lemma~\ref{lemma:DINe}, we conclude that for any $A\in SNe(D,\mathcal G)$, $A\in DINe(D,\mathcal G)$. Finally, by definition, every descendant inducing neighbors of $D$ is a descendant $D$ so $DINe(D,\mathcal G)\subset De(D,\mathcal G)$ (the inclusion is strict since children of $D$\footnote{That is, nodes $C$ \st $D\to C$} are descendants of $D$ which cannot be descendants inducing neighbors). 
\end{proof}

\subsection{Proof of Theorem~\ref{theorem:lmec}}

We restate Theorem~\ref{theorem:lmec} and then we prove it.

\mytheoremone*

\begin{proof}
Let $\mathcal G_i,\mathcal G_j\in LMEC(Y,h,\mathcal G)$. We prove every item of the theorem.:

\begin{enumerate}

\item[Item 1:] We proceed by induction on $h$. 

\begin{itemize}

    \item[\textit{Base case.}] For $h=0$, by definition:
    $
        NeHood(Y,0,\mathcal G_i)= \{Y\} = NeHood(Y,0,\mathcal G_j) .
    $

    \item[\textit{Induction hypothesis.}] Assume for some $h \ge 1$, 
    \[
        \mathcal H(h-1): \quad NeHood(Y,h-1,\mathcal G_i) = NeHood(Y,h-1,\mathcal G_j).
    \]

    \item[\textit{Inductive step.}] We show that the equality holds for $h$.
    Suppose, for the sake of contradiction, that there exists a node $D$ such that
    $
        D \in NeHood(Y,h,\mathcal G_i) \quad \text{and} \quad D \notin NeHood(Y,h,\mathcal G_j).
    $
    By definition of the neighborhood and $\mathcal H(h-1)$, there exists $A \in NeHood(Y,h-1,\mathcal G_i)=NeHood(Y,h-1,\mathcal G_j)$ such that 
    $
        D \in Ne(A,\mathcal G_i)$ and $D \notin Ne(A,\mathcal G_j).
    $
    Since $D \in Ne(A,\mathcal G_i)$, there exists no $\mathbb Z\subset \mathbb V$ such that $(A \indep D \mid \mathbb Z)_{\mathcal G_i}$.  
    On the other hand, for $D \notin Ne(A,\mathcal G_j)$, two cases arise:
    \begin{enumerate}
        \item $D \notin SNe(A,\mathcal G_j)$: then by definition of spurious neighbors, $\exists s^*=1,...,|\mathbb V|-2$ such that $D\notin \mathbb C_{s^*}(A,\mathcal G_j)$ meaning that $\exists \mathbb Z\subseteq \mathbb C_{s^*-1}(A,\mathcal G_i)$ of size $|\mathbb Z|=s^*$ satisfying $(D\indep A\mid \mathbb Z)_{\mathcal G_j}$. This contradicts the fact that $\mathcal G_i$ and $\mathcal G_j$ belongs to the same LMEC. 
        \item $D \in SNe(A,\mathcal G_j)$: then by Theorem~\ref{prop:spurious} $D \in De(A,\mathcal G_j)$ and for any $s\ge 0$, by taking $\mathbb Z := Pa(D,\mathcal G_j) \subseteq Ne(D,\mathcal G_j)\subseteq \mathbb C_{s-1}(D,\mathcal G_j)$ (see proof of Theorem~\ref{prop:spurious} for this last inclusion), we have $(A \indep D \mid \mathbb Z)_{\mathcal G_j}$, also a contradiction.
    \end{enumerate}

    In any case, we reach a contradiction, which proves that $\mathcal H(h-1)\implies \mathcal H(h)$ and thus for all $h$:
    $
        NeHood(Y,h,\mathcal G_i) = NeHood(Y,h,\mathcal G_j).
    $

\end{itemize}

From this point onward, whenever we write $NeHood(Y,h,\mathcal{G}_i)$, it is understood to denote the same set as $NeHood(Y,h,\mathcal{G}_j)$.

\item[Item 2:] The proof arguments for this item are quite similar to those used in the inductive step of item 1.  
We proceed by contradiction by assuming that there exist nodes $D, A \in NeHood(Y,h,\mathcal G_i)$ such that 
$
    D \in Ne(A,\mathcal G_i) \quad \text{and} \quad D \notin Ne(A,\mathcal G_j).
$
Since $D \in Ne(A,\mathcal G_i)$, there exists no $\mathbb Z \subset \mathbb V$ such that $(A \indep D \mid \mathbb Z)_{\mathcal G_i}.$
On the other hand, $D \notin Ne(A,\mathcal G_j)$ implies two possibilities:  
\begin{enumerate}
    \item $D \notin SNe(A,\mathcal G_j)$: then by definition of spurious neighbors, $\exists s^*\in\{1,...,|\mathbb V|-2\}$ such that $D\notin \mathbb C_{s^*}(A,\mathcal G_j)$ implying that $\exists \mathbb Z\subseteq \mathbb C_{s^*-1}(A,\mathcal G_i)$ of size $|\mathbb Z|=s^*$ satisfying $(D\indep A\mid \mathbb Z)_{\mathcal G_j}$, which is a contradiction.
    \item $D \in SNe(A,\mathcal G_j)$: by Theorem~\ref{prop:spurious}, $D$ is a descendant of $A$, and for any $s\ge 0$, by taking
    \(\mathbb Z = Pa(D,\mathcal G_j) \subseteq Ne(D,\mathcal G_j)\subseteq \mathbb C_s(D,\mathcal G_j)\), we have $ (A \indep D \mid \mathbb Z)_{\mathcal G_j},$
    again a contradiction.
\end{enumerate}

\item[Item 3: ] Let $D \in NeHood(Y,h,\mathcal G_i)$. First, we prove by induction that 
$
\forall s \ge 0, \mathbb C_s(D,\mathcal G_i) = \mathbb C_s(D,\mathcal G_j).
$

\begin{itemize}
\item[\textit{Base case.}]
For $s=0$, by definition we have
$
\mathbb C_0(D,\mathcal G_i) = \mathbb V \setminus \{D\} = \mathbb C_0(D,\mathcal G_j).
$

\item[\textit{Induction hypothesis.}]
Assume that for some $s \ge 1$,
\[
\mathcal H(s-1): \quad \mathbb C_{s-1}(D,\mathcal G_i) = \mathbb C_{s-1}(D,\mathcal G_j).
\]

\item[\textit{Induction step.}]
We now show that the equality also holds for $s$:
\begin{align*}
\mathbb C_s(D,\mathcal G_i)
&= \Big\{
A \in \mathbb C_{s-1}(D,\mathcal G_i) \;\Big|\;
\nexists\, \mathbb Z \subseteq \mathbb C_{s-1}(D,\mathcal G_i)\setminus\{A\}: \;
|\mathbb Z| = s\land
(D \indep A \mid \mathbb Z)_{\mathcal G_i}
\Big\} \\[0.5em]
&= \Big\{
A \in \mathbb C_{s-1}(D,\mathcal G_j) \;\Big|\;
\nexists\, \mathbb Z \subseteq \mathbb C_{s-1}(D,\mathcal G_j)\setminus\{A\}: \;
|\mathbb Z| = s\land
(D \indep A \mid \mathbb Z)_{\mathcal G_i}
\Big\} \text{ by $\mathcal H(s-1)$} \\[0.5em]
&= \Big\{
A \in \mathbb C_{s-1}(D,\mathcal G_j) \;\Big|\;
\nexists\, \mathbb Z \subseteq \mathbb C_{s-1}(D,\mathcal G_j)\setminus\{A\}: \;
|\mathbb Z| = s\land
(D \indep A \mid \mathbb Z)_{\mathcal G_j}
\Big\} \text{by def.~\ref{def:lmec}} \\[0.5em]
&= \mathbb C_s(D,\mathcal G_j).
\end{align*}

Hence, $\mathcal H(s-1) \Rightarrow \mathcal H(s)$, and
$\forall s \ge 0, \quad \mathbb C_s(D,\mathcal G_i) = \mathbb C_s(D,\mathcal G_j)$.

\end{itemize}

Since $Ne(D,\mathcal G_i) \subseteq \mathbb C_s(D,\mathcal G_i)$ (see the proof of Theorem~\ref{prop:spurious}),  
and by definition of spurious neighbors,
$
\mathbb C_{|\mathbb V|-2}(D,\mathcal G_i) = Ne(D,\mathcal G_i) \cup SNe(D,\mathcal G_i),
$
it follows that
\[
\mathbb C_s(D,\mathcal G_i) = \mathbb C_s(D,\mathcal G_j)
\quad \Rightarrow \quad
Ne(D,\mathcal G_i) \cup SNe(D,\mathcal G_i)
= Ne(D,\mathcal G_j) \cup SNe(D,\mathcal G_j).
\]

\item[Item 4:] 
Let $D_1, D_2, D_3 \in NeHood(Y,h,\mathcal{G}_i)$. 
Without loss of generality, suppose there exists in $\mathcal{G}_i$ a UC of the form 
$D_1 \to D_2 \leftarrow D_3$, and we aim to show that the same configuration necessarily exists in $\mathcal{G}_j$. By definition of a UC, this implies the existence of a separating set 
$\mathbb{S}$ between $D_1$ and $D_3$, such that for some $s\in \{0,...,|\mathbb V|-2\}$,
$\mathbb{S} \subseteq Ne(D_1,\mathcal G_i)\subseteq \mathbb C_{s}(D_1,\mathcal G_i)$ or $\mathbb{S} \subseteq Ne(D_3,\mathcal G_i)\subseteq \mathbb C_{s}(D_3,\mathcal G_i)$, and 
$D_2 \notin \mathbb{S}$. From items~(1) and~(2) the same triple is unshielded in $\mathcal G_j$ and by definition of the LMEC, the same set $\mathbb S$ must satisfy $(D_1 \indep D_3 \mid \mathbb{S})_{\mathcal{G}_j}$. 
Assume, for contradiction, that this unshielded triple is not a collider in $\mathcal{G}_j$. 
Then it must be either a chain ($D_1 \to D_2 \to D_3$ or $D_1 \leftarrow D_2 \leftarrow D_3$) 
or a fork ($D_1 \leftarrow D_2 \to D_3$), both of which are active paths that can only be blocked 
by conditioning on the middle node $D_2$. 
This would imply $D_2 \in \mathbb{S}$, contradicting $D_2 \notin \mathbb{S}$. Therefore, the unshielded collider $D_1 \to D_2 \leftarrow D_3$ must also exist in $\mathcal{G}_j$.

\item[Item 5:]Assume that for all $W\in \mathbb W_D$, the triple $(D,A,W)$ \textbf{is not a non-collider} in $\mathcal G_i$, 
and that there exists $W\in \mathbb W_D$ such that $(D,A,W)$ \textbf{is a non-collider} in $\mathcal G_j$. 
We show that this assumption leads to a contradiction.

First, note that $W\in \mathbb W_D$
implies that, for some $s^*\ge 1$, there exists 
$\mathbb S\subseteq \mathbb C_{s^*-1}(D,\mathcal G_i)$ such that $|\mathbb S|=s^*$ and $(D\indep W\mid \mathbb S)_{\mathcal G_i}.$
From the previously established results and by the definition of the LMEC, 
\textbf{the same set} $\mathbb S$ also satisfies $(D\indep W\mid \mathbb S)_{\mathcal G_j}$.  
Moreover, since $(D,A,W)$ forms a non-collider triple in $\mathcal G_j$, 
this implies that $A\in \mathbb Z$ for every $\mathbb Z$  
$(D\indep W\mid \mathbb Z)_{\mathcal G_j}$; otherwise, $D$ and $W$ could not be $d$-separated. Then, in particular, our assumptions imply that $A\in \mathbb S$. 

We now examine all possible configurations of the \textbf{active} triple $(D,A,W)$ in $\mathcal G_i$, 
and show that our assumption always leads to a contradiction.

\begin{enumerate}
\item \textit{Case 1:} $A\in Ne(D,\mathcal G_i)$ and $W\in Ne(A,\mathcal G_i)$.\\
Since $(D,A,W)$ is \textbf{not} a non-collider in $\mathcal G_i$, the only configuration compatible with this situation
is the collider structure
$D\to A\leftarrow W.$
In this case, $A$ is a collider, and therefore, for any separating set $\mathbb Z$ between $D$ and $W$, 
we have $A\notin \mathbb Z$. Then, in particular, $A\notin \mathbb S$, 
which contradicts that $A\in \mathbb S$.

\item \textit{Case 2:} $A\in Ne(D,\mathcal G_i)$ and $W\notin Ne(A,\mathcal G_i)$.\\
Two subcases arise:
\begin{enumerate}
\item \emph{No active path between $D$ and $W$ passes through $A$.}\\
Then, there exists a separating set $\mathbb Z$ of $D$ and $W$ in $\mathcal G_i$ not containing $A$. Taking $\mathbb S=\mathbb Z$ leads to $A\notin \mathbb S$ and contradicts our assumption.
\item \emph{There exists an active path} $\pi = \langle D, A, B, \cdots, W\rangle$ \emph{containing $A$.}\\
In that case, the triple $(D,A,B)$ must form either a non-collider triple, 
otherwise the path $\pi$ would not be active.  
Two situations may occur:
  \begin{itemize}
  \item If $B\in Ne(D,\mathcal G_j)\cup SNe(D,\mathcal G_j)$, then $B\in \mathbb C_{|\mathbb V|-2}$, 
        implying $B\in \mathbb C_{s^*-1}$. The path can thus be blocked by conditioning on $B$ 
        instead of $A$, yielding a separating set $\mathbb S$ with $A\notin \mathbb S$, 
        which is contradictory.
  \item If $B\notin Ne(D,\mathcal G_j)\cup SNe(D,\mathcal G_j)$, 
        then, since $(D,A,B)$ must form a non-collider triple in $\mathcal G_j$, it directly contradicts our assumption that $\forall W\in \mathbb W_D$, $(D,A,W)$ is \textbf{not} a non-collider in $\mathcal G_i$.  
  \end{itemize}
In both subcases, we reach a contradiction.
\end{enumerate}

\item \textit{Case 3:} $A\notin Ne(D,\mathcal G_i)$.\\
Then $A\in SNe(D,\mathcal G_i)$, hence $A\in De(D,\mathcal G_i)$ by Theorem~\ref{prop:spurious}.  
Thus, any active path between $D$ and $W$ passing through $A$ is either a backdoor path that can be blocked by conditioning on the parents of $D$ excluding $A$, or a directed path
$D\to B\ \to C\to \cdots \to A\to \cdots \to W$.
Two possibilities arise on these directed paths:
\begin{itemize}
\item If $C\in Ne(D,\mathcal G_i)\cup SNe(D,\mathcal G_i)$, 
then the directed path between $D$ and $W$ passing through $A$ 
can be blocked by conditioning on $C$. 
If it is true for all of the directed paths, there exists a separating set $\mathbb S$ between $D$ and $W$ such that $A\notin \mathbb S$, contradiction.
\item If $C\notin Ne(D,\mathcal G_i)\cup SNe(D,\mathcal G_i)$, 
then $(D,B,C)$ forms a chain, which is impossible by assumption.
\end{itemize}
In both cases, a contradiction follows.
\end{enumerate}

\noindent
 In every possible configuration, the initial assumption leads to a contradiction. 
Hence, there cannot exist $W\in \mathbb W_D$ such that $(D,A,W)$ is not a non-collider in $\mathcal G_j$ 
if it is a non-collider in $\mathcal G_i$.
\end{enumerate}
\end{proof}

\subsection{Proof of Theorem~\ref{theorem:leg}}

We proceed to prove the theorem after having restated it.

\mytheoremtwo*

\begin{proof}
Items 1, 2, 3, 4 and 6 of Theorem~\ref{theorem:leg} follow directly from Theorem~\ref{theorem:lmec}, which ensures that the edges of types $(-)$, $(\to)$, and $(\arrowdoublebar)$ introduced by these items are consistent with Definition~\ref{def:leg}. Moreover, item 5 guarantees that no new unshielded collider common to all DAGs in the LMEC (as stated in Theorem~\ref{theorem:lmec}) is introduced and that no cycles are created ensuring that $\mathcal L^{Y,h}$ is a partially directed acyclic graph. Therefore, the directed edges added by these rules must also comply with Definition~\ref{def:leg}.
\end{proof}

\mycorollaryone*

\begin{proof}
If $\mathcal{G}_1$ and $\mathcal{G}_2$ belong to the same LMEC, then Theorem~\ref{theorem:lmec} ensures that all the structural features required for constructing the LEG, as defined in Theorem~\ref{theorem:leg}, are shared by both DAGs. Consequently, they have the same LEG.
\end{proof}

\subsection{Proof of Theorem~\ref{theorem:faithfulness_implications}}

\mytheoremthree*

\begin{proof}
We will show that, under the assumptions of Theorem~\ref{theorem:faithfulness_implications}, the inferred LEG $\widehat{\mathcal L}$ coincides with $\mathcal L^{Y,h}$ by proving that $\widehat{\mathcal L}$ satisfies all the conditions of Theorem~\ref{theorem:leg}. Throughout, we will indicate explicitly when the local faithfulness assumption is required. The causal sufficiency assumption is necessary whenever we consider that two non-adjacent nodes $N_1$ and $N_2$ in $\mathcal G$ possess a separating set.

First, we prove by induction on $s$ that, for any $D \in NeHood(Y,h,\mathcal G)$, the set of nodes connected to $D$ at iteration $s$ of the \textbf{LocPC}  algorithm, $\mathbb C_s(D,\widehat{\mathcal L})$, coincides with $\mathbb C_s(D,\mathcal G)$.

\textbf{Base case ($s=0$)}: By construction, \textbf{LocPC}  initially connects all edges to $D$, hence
\[
\mathbb C_0(D,\widehat{\mathcal L}) = \mathbb V \setminus \{D\} = C_0(D,\mathcal G).
\]

\textbf{Inductive hypothesis}: Assume that for some $s-1 \ge 0$, 
$
\mathbb C_{s-1}(D,\widehat{\mathcal L}) = C_{s-1}(D,\mathcal G).
$

\textbf{Inductive step}: For $s \ge 1$, \textbf{LocPC}  tests, for each node $N \in \mathbb C_{s-1}(D,\widehat{\mathcal L})$, whether there exists a subset $\mathbb Z \subseteq \mathbb C_{s-1}(D,\widehat{\mathcal L})$ of size $s$ such that
$
(D \indep N \mid \mathbb Z)_{\mathcal P}.
$
If such a set exists, the edge between $D$ and $N$ is removed from $\mathbb C_s(D,\widehat{\mathcal L})$. Under the assumption of \emph{local faithfulness} and with access to a CI oracle,
$
(D \indep N \mid \mathbb Z)_{\mathcal P} \iff (D \indep N \mid \mathbb Z)_{\mathcal G},
$
which implies
$
\mathbb C_s(D,\widehat{\mathcal L}) = C_s(D,\mathcal G).
$

Hence, by induction, the property holds for all $s \ge 0$.

We use this result to prove by induction on $h$ that, under the assumptions of the theorem, \textbf{LocPC}  correctly identifies the $h$-neighborhood of $Y$, that is,
$
NeHood(Y,h,\widehat{\mathcal L}) = NeHood(Y,h, \mathcal L^{Y,h}).
$

\textbf{Base case ($h=0$):}  
$
NeHood(Y,0,\widehat{\mathcal L}) = \{Y\} = NeHood(Y,0, \mathcal L^{Y,h}).
$

\textbf{Inductive hypothesis:} Assume that for $h-1 \ge 0$,
$
NeHood(Y,h-1,\widehat{\mathcal L}) = NeHood(Y,h-1, \mathcal L^{Y,h}).
$

\textbf{Inductive step:}  
For each $D \in NeHood(Y,h-1,\widehat{\mathcal L})$, we have shown that
$
\mathbb C_{|\mathbb V|-2}(D,\widehat{\mathcal L}) = C_{|\mathbb V|-2}(D,\mathcal G).
$
By definition of $\mathbb C_s$, this implies
$
Ne(D,\widehat{\mathcal L}) \cup SNe(D,\widehat{\mathcal L}) = Ne(D,\mathcal G) \cup SNe(D,\mathcal G)
$
at iteration $h-1$ (see the proof of Theorem~\ref{theorem:lmec} for details). At this iteration, the algorithm continues exploring adjacent nodes of $D$ that have not yet been discovered, i.e., nodes $N \in Ne(D,\widehat{\mathcal L}) \cup SNe(D,\widehat{\mathcal L})$ that are unvisited. If $N \in Ne(D,\mathcal G)$, no separating set exists and \textbf{LocPC}  will not remove the edge between $D$ and $N$. If $N \in SNe(D,\mathcal G)$, then $N \in De(D,\mathcal G)$ by Theorem~\ref{prop:spurious}, and by $d$-separation and the Markov property, there exists some $s = 1,\cdots,|\mathbb V|-2$ and a set $\mathbb Z \subseteq Pa(N,\mathcal G) \subseteq Ne(N,\mathcal G) \subset \mathbb C_s(N,\mathcal G)$ such that
$
(D \indep N \mid \mathbb Z)_\mathcal P,
$
and the edge between $D$ and $N$ is removed.

Consequently, each node in the $(h-1)$-neighborhood retains edges only to its true neighbors. The union of these nodes with the newly discovered adjacent nodes forms the $h$-neighborhood, which is therefore correctly identified by \textbf{LocPC}  at iteration $h$, establishing that
$
NeHood(Y,h,\widehat{\mathcal L}) = NeHood(Y,h, \mathcal L^{Y,h}).
$

Beyond the correct identification of the neighborhood in $\widehat{\mathcal L}$ (item 1 of Theorem~\ref{theorem:leg}), the previous proof also contains the elements establishing items 2 and 4 of Theorem~\ref{theorem:leg}. We now show that, under the assumptions of the theorem, the CI tests allow \textbf{LocPC}  to correctly identify the unshielded colliders (UCs) among nodes in the $h$-neighborhood.  

Once the LEG skeleton is constructed and correct, \textbf{LocPC}  identifies unshielded triples $D_1,D_2,D_3 \in NeHood(Y,h,\widehat{\mathcal L})$. Since $D_1 \notin Ne(D_3,\widehat{\mathcal L})$, \textbf{LocPC}  has found a separating set $\mathbb S$ such that 
$
(D_1 \indep D_3 \mid \mathbb S)_\mathcal P.
$
If $D_1 \to D_2 \leftarrow D_3$ in $\mathcal G$, then $D_2 \notin \mathbb S$, otherwise this path would be activated by conditioning on $D_2$, and the CI would not hold. If the triple is a non-collider in $\mathcal G$, then it is necessary that some $A \in \mathbb S$ to render $D_1$ and $D_3$ conditionally independent given $\mathbb S$. Under local faithfulness, \textbf{LocPC}  can thus correctly orient all UCs in the $h$-neighborhood based on the performed independence tests, establishing item 3 of Theorem~\ref{theorem:leg}.  

The application of Meek's rules is identical in \textbf{LocPC}  and in Theorem~\ref{theorem:leg}, ensuring the satisfaction of item 5 of Theorem~\ref{theorem:leg}.

It remains to verify the last item of Theorem~\ref{theorem:leg}, which is ensured by NNC Rule of Definition~\ref{def:local_rules}. Let $D \in NeHood(Y,h,\mathcal G)$ and $A \notin NeHood(Y,h,\mathcal G)$ with $D-A \in \widehat{\mathcal L}$. \textbf{LocPC}  examines each $W \notin \{NeHood(Y,h,\mathcal{G}) \cup Ne(D,\mathcal{G}) \cup SNe(D,\mathcal{G})\}$ and checks whether $A \in \mathbb S_W$, where $\mathbb S_W$ is the separating set found to separate $W$ and $D$. If $A \notin \mathbb S_W$ for all such $W$, then there is no non-collider triple $(D,A,W)$ in $\mathcal G$, since otherwise $A$ would always be required to separate $W$ and $D$. Therefore, adding $D \arrowdoublebar A$ is correct, confirming that item 6 of Theorem~\ref{theorem:leg} is satisfied by $\widehat{\mathcal L}$.  

In conclusion, we have shown that under the assumptions of local faithfulness and causal sufficiency, and with access to perfect CI information, it holds that
$
\widehat{\mathcal L} = \mathcal L^{Y,h},
$
thus establishing the correctness of \textbf{LocPC} .
\end{proof}

\subsection{Proof of complexity}

First, we state and prove a lemma that will be useful for establishing the complexity result:

\begin{lemma}[Maximal Conditioning Set Size in \textbf{LocPC} ]
\label{lemma:size_cond_set}
Let $D \in NeHood(Y,h,\mathcal G)$, $W$ be any node, \(k_d := \max\{|Ne(D,\mathcal G)| : D \in NeHood(Y,h,\mathcal G)\}\), \(k_i := \max\{|DINe(D,\mathcal G)| : D \in NeHood(Y,h,\mathcal G)\}\) and \(k_\ell := k_d + k_i\). If \textbf{LocPC}  identifies a separating set $\mathbb{S}$ for $D$ and $W$, then it necessarily satisfies $|\mathbb{S}| \le k_\ell$.
\end{lemma}

\begin{proof}
    Let \(D\in NeHood(Y,h,\mathcal G)\) and let \(W\) be any node. We first show that every separating set \(\mathbb S\) returned by \textbf{LocPC}  with hop \(h\) satisfies \(|\mathbb S|\le k_\ell:=k_d+k_i\). Assume for contradiction that \(|\mathbb S|>k_\ell\). As in any PC-style procedure, \textbf{LocPC}  always searches for a \emph{smallest} separating set among the subsets it considers. If \(W\in NeHood(Y,h,\mathcal G)\), then by the Markov property there exists \(\mathbb Z\subseteq Pa(D)\) or \(\mathbb Z\subseteq Pa(W)\) such that \((D\indep  W\mid \mathbb Z)_\mathcal G\). Since parents are neighbors, \textbf{LocPC}  finds such a \(\mathbb Z\) at iteration \(h\), with \(|\mathbb Z|\le k_d\le k_\ell<|\mathbb S|\), contradicting the smallest condition. If \(W\notin NeHood(Y,h,\mathcal G)\), consider three subcases.  
(i) If \(W\notin  DINe(D)\), Lemma~\ref{lemma:DINe} ensures the existence of \(\mathbb Z\subseteq Ne(D)\) such that \((W\indep  D\mid \mathbb Z)_\mathcal G\) and \(|\mathbb Z|\le k_d<|\mathbb S|\), a contradiction.  
(ii) If \(W\in DINe(D)\), then by definition there exists a directed chain \(\pi_c: D\to\cdots\to W\). Since we assumed that \textbf{LocPC}  found a separating set $\mathbb S$, this chain must be blocked. However, by Lemma~\ref{lemma:DINe} no neighbor of \(D\) can block it, and every non-neighbor of $D$ on \(\pi_c\) is in \(DINe(D)\). By blocking all DIP paths using nodes in \(DINe(D)\) (at most \(k_i\)) and all remaining paths using neighbors of \(D\) (at most \(k_d\)), we construct a set \(\mathbb Z\subseteq Ne(D)\cup DINe(D)\) with \(|\mathbb Z|\le k_d+k_i=k_\ell<|\mathbb S|\), a contradiction. Thus any separating set \(\mathbb S\) found by \textbf{LocPC}  with hop \(h\) must satisfy \(|\mathbb S|\le k_\ell:=k_d+k_i\).\end{proof}

Now, restate and proof Theorem~\ref{theorem:complexity}. 

\mytheoremcomplexity

\begin{proof}
When exploring a node \(D\), \textbf{LocPC}  searches for a separating set between \(D\) and the \(n-1\) other nodes by increasing the conditioning set size \(s\). At each iteration \(s\), in the worst case, all subsets of size \(s\) among the remaining \(n-2\) nodes are tested. From Lemma~\ref{lemma:size_cond_set}, the largest examined conditioning set has size at most \(k_\ell\), and that after this procedure the remaining true and spurious neighbors are at most \(k_d+k_i=k_\ell\) (with at most \(k_i\) spurious neighbors by Theorem~\ref{prop:spurious}), the procedure terminates once no unexplored candidate sets remain. This yields a worst-case cost of \((n-1)\sum_{s=0}^{k_\ell}\binom{n-2}{s}\) for each node discovered, increasing with the hop \(h\). For any \(h>0\), the procedure is repeated on nodes still connected to the target, at most \(k_\ell=k_d+k_i\) (true neighbors \(k_d\) plus at most \(k_i\) spurious neighbors). If \(h>1\), spurious neighbors are pruned, and only true neighbors (\(\le k_d\)) are explored. This process repeats for each hop up to \(h\), first exploring potential neighbors (\(\le k_\ell\)), then continuing only with true neighbors (\(\le k_d\)), yielding a total of at most \(1+\sum_{j=0}^{h-1} k_d^j k_\ell = 1 + k_\ell \frac{1-k_d^h}{1-k_d}\) repetition.

Finally, for the complexity as $n \to \infty$ with fixed $k_\ell$, $k_d$, and $h$, and 
$\left(1 + k_\ell \frac{1 - k_d^h}{1 - k_d}\right) < n$ (which is natural: this quantity is precisely the number of visited nodes, and if the task required exploring all the nodes, relying on \emph{local} causal discovery would no longer be meaningful):

\[
\left(1 + k_\ell \frac{1 - k_d^h}{1 - k_d}\right)(n-1)\sum_{s=0}^{k_\ell} \binom{n-2}{s} 
\le \left(1 + k_\ell \frac{1 - k_d^h}{1 - k_d}\right)(n-1)(k_\ell+1) \frac{(n-2)^{k_\ell}}{k_\ell!}=  \mathcal{O}(n^{k_\ell+1})
\]

This concludes the proof of the complexity.
\end{proof}

\subsection{Proof of Theorem~\ref{corollary:identifiability_CDE}}

Before proving Theorem~\ref{corollary:identifiability_CDE}, we state and prove Theorem~\ref{th:orientability} which is useful to prove Theorem~\ref{corollary:identifiability_CDE}.

\begin{theorem}
\label{th:orientability}
Let $\mathbb D \subseteq NeHood(Y, h, \mathcal G)$, and let $\mathcal{L}^{Y,h}=(\mathbb V,\mathbb E^{Y,h})$ denote the LEG with $h \ge 1$.  
If $\mathbb D$ satisfies the non-orientability criterion, then for all $D_i,D_j \in \mathbb D$, $(D_i-D_j) \in \mathbb E^{Y,h} \implies (D_i-D_j) \in \mathbb E^{Y,k} \quad \forall k>h.$
\end{theorem}

\begin{proof}
After constructing the LEG $\mathcal{L}^{Y,h}$ for a given hop $h$, the only way the undirected edge $D_i - D_j$ could become oriented in a LEG $\mathcal{L}^{Y,k}$ with $k > h$ is through the propagation of orientations via Meek's rules. These are the only mechanisms capable of orienting edges already present at hop $h$.

To prevent such propagation, we first require that there are no undirected edges between nodes in $\mathbb{D}$ and nodes outside $\mathbb{D}$; such connections could serve as pathways for orientation propagation under Meek’s rules.

Furthermore, for each node $D \in \mathbb{D}$, if there exists at most one node $A \notin \mathbb{D}$ such that $D \arrowdoublebar A$, then we ensure the following:\footnote{If multiple such nodes $A$ exist, they could form an unshielded collider with $D$ in some LEG of hop $k > h$, possibly leading to orientations like $A \to D$, and thus allowing orientations to propagate within $\mathbb{D}$ via Meek's rules.}

\begin{itemize}
    \item \textbf{Case $A \in Ne(D,\mathcal G)$:}
    \begin{itemize}
        \item Either no new neighbor of $A$ is discovered (i.e., no node not already included at hop $h$), in which case the edge $D - A$ cannot be oriented and no propagation occurs;
        \item Or, if new neighbors of $A$ are discovered (which are non-neighbors of $D$), then the edge $D \to A$ will be oriented. However, this orientation is incompatible with all of Meek’s rules for propagating directions back into $\mathbb D$.
    \end{itemize}

    \item \textbf{Case $A \notin Ne(D,\mathcal G)$:}  
    This implies $A \in SNe(D,\mathcal G)$ (since $D\arrowdoublebar A$ implies that $A\in Ne(D,\mathcal G)\cup SNe(D,\mathcal G)$), and thus at the next iteration (hop $h+1$) the edge between $D$ and $A$ will be removed. Consequently, no orientation information can propagate into $\mathbb D$ through this edge.
\end{itemize}

Therefore, if the non-orientability conditions (Def.~\ref{def:orientability}) are satisfied, no orientation can reach the edge $D_i - D_j$ in $\mathcal{L}^{Y,k}$ for any $k > h$. As a result, $D_i - D_j$ remains undirected in all subsequent LEGs.\end{proof}

We now turn to the proof of Theorem~\ref{corollary:identifiability_CDE} which follows directly from of Theorem~\ref{th:orientability}.

\mycorollarytwo*

\begin{proof}
    First, note that the essential graph $\mathcal{C} = \mathcal{L}^{Y, k_{\mathrm{max}}}$ where $k_{\mathrm{max}}$ is the length of the longest path from $Y$ to any other node in the graph. By applying Theorem~\ref{th:orientability} to the undirected edges adjacent to $Y$ that are included in the subset $\mathbb{D}$, we deduce that these edges remain undirected in the essential graph $\mathcal{C}$. Then, by Theorem~5.4 of \citep{Flanagan_2020}, it follows that the CDEs on $Y$ are not identifiable.
\end{proof}

\subsection{Proof of Theorem~\ref{th:sound_and_complete-LocPC-CDE}}

\mytheoremsix*

\begin{proof}
Assume that the CDE is identifiable. This means that in the essential graph $\mathcal{C}$, all nodes adjacent to $Y$ are oriented. Under these assumptions, and assuming access to a perfect CI oracle, it follows from Theorem~\ref{theorem:faithfulness_implications} that at each iteration over $h$, the locally estimated essential graph (LEG) $\widehat{\mathcal{L}}$ discovered by \textbf{LocPC} is correct. Moreover, if the CDE is identifiable, then by Corollary~\ref{corollary:identifiability_CDE}, the non-orientability criterion (Def.~\ref{def:orientability}) will never be satisfied. Consequently, the local discovery will proceed until all edges adjacent to $Y$ are oriented. This will eventually occur since, in the worst case, the discovered graph $\widehat{\mathcal{L}}$ becomes equal to the essential graph $\mathcal{C}$ if all relevant nodes are included. Therefore, there exists an iteration in which all edges adjacent to $Y$ are oriented, and the algorithm will conclude that the CDEs with respect to $Y$ are identifiable. The returned LEG necessarily has all of $Y$'s adjacents oriented, which is sufficient for estimating the CDE.

Assume that the CDE is not identifiable. This implies that in the essential graph $\mathcal{C}$, the adjacency of $Y$ is not fully oriented. Under these assumptions, and given a CI oracle, the LEG $\widehat{\mathcal{L}}$ discovered at each iteration $h$ is correct by Theorem~\ref{theorem:faithfulness_implications}. According to Corollary~\ref{corollary:identifiability_CDE}, if the non-orientability criterion (Def.~\ref{def:orientability}) is satisfied—which acts as a stopping condition—then the CDE is not identifiable. If the stopping condition is never triggered, the algorithm continues the local discovery process until it recovers the full essential graph $\mathcal{C}$. Thus, in all cases, when the algorithm terminates and finds that $Y$'s adjacency is not fully oriented, we are guaranteed that the CDE is indeed not identifiable.\end{proof}

\subsection{Proof of Theorem~\ref{corollary:optimality_LocPC-CDE}}

\mypropositiontwo*

\begin{proof}
    Assume that the CDE is not identifiable and that the algorithm declares non-identifiability at iteration $\widehat{h}$. We now show that it was not possible to conclude non-identifiability at iteration $\widehat{h} - 1$. If the algorithm did not terminate at iteration $\widehat{h} - 1$, this implies that the non-orientability condition (Definition~\ref{def:orientability}) was not satisfied at that stage. However, as demonstrated in the proof of Theorem~\ref{th:orientability}, when the non-orientability condition is not satisfied, it remains possible that an edge in the discovery set $\mathbb{D}$ becomes oriented as the size of the discovery increases. Therefore, it follows directly that without continuing the discovery at iteration $\widehat{h}$, it was not possible to be certain that the CDE was not identifiable.
\end{proof}

\section{About spurious neighbors and descendant inducing paths}
\label{appendix:spurious}

\subsection{A simple example of spurious neighbor}
We illustrate here the simplest example of a \emph{spurious edge} introduced by a localized PC algorithm. Let $\mathcal{G}$ be the underlying DAG and $D$ the target. The steps of the a PC-style local discovery algorithm around $D$ are as follows to build the skeleton:

\begin{enumerate}
    \item \textbf{Initialization:} all edges to $D$ are connected, $\mathbb{C}_0(D,\mathcal{G}) = \{A,B,C\}$.
    \item \textbf{Conditioning with size $s=0$:} unconditional independencies are tested. Since $(D \indep C \mid \emptyset)_\mathcal{G}$, the edge $(D,C)$ is removed, yielding $\mathbb{C}_1(D,\mathcal{G}) = \{A,B\}$.
    \item \textbf{Conditioning with size $s=1$:}  
    \begin{itemize}
        \item For $A$, no separating set exists among $\mathbb{C}_1(D,\mathcal{G})$, so the edge $(D,A)$ remains.  
        \item For $B$, the required separator would be $\{A,C\}$, but $C$ was previously removed. No set among $\mathbb{C}_1(D,\mathcal{G})$ separates $D$ and $B$, so $(D,B)$ remains.
    \end{itemize}
\end{enumerate}

Thus, after a PC run targeted at $D$, $B$ becomes a spurious neighbor according to def.~\ref{def:sne}:
\[
\mathbb{C}_2(D,\mathcal{G}) = \{A,B\}, \quad Ne(D,\mathcal{G}) = \{A\} \implies SNe(D,\mathcal{G}) = \{B\}.
\]

\begin{center}
\begin{minipage}{0.23\textwidth}
\centering
\begin{tikzpicture}[{black, circle, draw, inner sep=0}]
\tikzset{nodes={draw,rounded corners, minimum height=0.6cm, minimum width=0.6cm}}
\node (D)  at (0,0) [fill=gray!30] {$D$};
\node (A)  at (0,1) {$A$};
\node (B) at (1,0) {$B$};
\node (C) at (1,1) {$C$};
\draw[->, >=latex] (D) -- (A);
\draw[->, >=latex] (A) -- (B);
\draw[->, >=latex] (C) -- (B);
\draw[->, >=latex] (C) -- (A);
\end{tikzpicture}

$\mathcal G$
\end{minipage}
\hfill
\begin{minipage}{0.23\textwidth}
\centering
\begin{tikzpicture}[{black, circle, draw, inner sep=0}]
\tikzset{nodes={draw,rounded corners, minimum height=0.6cm, minimum width=0.6cm}}
\node (D)  at (0,0) [fill=gray!30] {$D$};
\node (A)  at (0,1) {$A$};
\node (B) at (1,0) {$B$};
\node (C) at (1,1) {$C$};
\draw[-, >=latex] (D) -- (A);
\draw[-, >=latex] (D) -- (B);
\draw[-, >=latex] (D) -- (C);
\end{tikzpicture}

Step 1

\end{minipage}
\hfill
\begin{minipage}{0.23\textwidth}
\centering
\begin{tikzpicture}[{black, circle, draw, inner sep=0}]
\tikzset{nodes={draw,rounded corners, minimum height=0.6cm, minimum width=0.6cm}}
\node (D)  at (0,0) [fill=gray!30] {$D$};
\node (A)  at (0,1) {$A$};
\node (B) at (1,0) {$B$};
\node (C) at (1,1) {$C$};
\draw[-, >=latex] (D) -- (A);
\draw[-, >=latex] (D) -- (B);
\end{tikzpicture}

Step 2

\end{minipage}
\hfill
\begin{minipage}{0.23\textwidth}
\centering
\begin{tikzpicture}[{black, circle, draw, inner sep=0}]
\tikzset{nodes={draw,rounded corners, minimum height=0.6cm, minimum width=0.6cm}}
\node (D)  at (0,0) [fill=gray!30] {$D$};
\node (A)  at (0,1) {$A$};
\node (B) at (1,0) {$B$};
\node (C) at (1,1) {$C$};
\draw[-, >=latex] (D) -- (A);
\draw[-, >=latex, red] (D) -- (B);
\end{tikzpicture}

Following steps

\end{minipage}
\end{center}

The approach developed in this paper shows that, in this case, the spurious edge between $D$ and $B$ will be removed when targeting $B$ in a future iteration of the local-to-global causal discovery procedure (for instance, in our \textbf{LocPC} algorithm). Repeating the same steps will reveal that $\{A,C\}$ separates $D$ and $B$, and the edges to $A$ and $C$ will never be removed when targeting $B$, since they are true neighbors.

\subsection{Descendant inducing paths}

We showed in Theorem~\ref{prop:spurious} that all such spurious relations arise from a specific type of path, which we call \emph{descendant inducing paths} (DIPs), defined in Definition~\ref{def:dip}. DIPs are particular types of inducing paths as introduced in~\cite{Spirtes_2000}, both in their definition and in their properties. We recall that an inducing path $\pi$ between $A$ and $B$ relative to a subset of nodes $\mathbb O$ is an undirected path such that every node in $\mathbb O$, except the endpoints, is a collider on $\pi$, and every collider on $\pi$ is an ancestor of $A$ or $B$. The differences between inducing paths and DIPs are therefore twofold:
(i) a certain directionality is imposed on the path (a DIP is defined between nodes $A$ and $B$ with an asymmetric role for $A$ and $B$), which facilitates the definition of descendant inducing neighbors, an asymmetric relation; and
(ii) a DIP is defined explicitly relative to a subset of the neighbors of $A$, rather than to an arbitrary subset of nodes on the path. Trivially, every DIP is then an inducing path but not every inducing path is a DIP (an example is provided below). Moreover, an inducing path relative to a set $\mathbb S$ is a path that remains active regardless of the subset of $\mathbb O$ on which we condition. A DIP satisfies an analogous property, stated in Lemma~\ref{lemma:DINe}, taking into account that local PC approaches can, in the worst case, condition only on the neighbors of the target node.

Below, we provide an illustration of these paths to complement the formal definition.

\begin{center}
\begin{tikzpicture}[{black, circle, draw, inner sep=0}]
\tikzset{nodes={draw,rounded corners, minimum height=0.6cm, minimum width=0.6cm}}
\node (D)  at (0,0) [fill=gray!30] {$D$};
\node (A)  at (1,0) {$A$};
\node (B)  at (2,0) {$B$};
\node (C)  at (1,-1) {$C$};
\node (E)  at (3,0) {$E$};
\node (F)  at (4,0) {$F$};
\node (G)  at (3,-1) {$G$};
\node (H)  at (-1,0) {$H$};
\node (I)  at (-2,0) {$I$};
\node (J)  at (-3,0) {$J$};
\node (K)  at (-2,-1) {$K$};
\node (L)  at (0,-1) {$L$};
\node (M)  at (-1,-1) {$M$};
\node (N)  at (-1,-2) {$N$};
\node (O)  at (0,-2) {$O$};

\draw[->, >=latex] (D) -- (A);
\draw[->, >=latex] (A) -- (B);
\draw[->, >=latex] (B) -- (E);
\draw[->, >=latex] (E) -- (F);
\draw[->, >=latex] (C) -- (A);
\draw[->, >=latex] (C) -- (B);
\draw[->, >=latex] (G) -- (E);
\draw[->, >=latex] (G) -- (F);
\draw[->, >=latex] (D) -- (H);
\draw[->, >=latex] (H) -- (I);
\draw[->, >=latex] (K) -- (I);
\draw[->, >=latex] (K) -- (J);
\draw[->, >=latex] (I) -- (J);
\draw[->, >=latex] (M) -- (L);
\draw[->, >=latex] (D) -- (L);
\draw[->, >=latex] (M) -- (N);
\draw[->, >=latex] (O) -- (N);
\draw[->, >=latex] (L) -- (O);
\draw[->, >=latex, bend left] (D) to (E);
\end{tikzpicture}
\end{center}

The path $\pi_1 =  D \to A \leftarrow C \to B $ is a DIP with respect to $\mathbb{L}_1 = \{D, A, B\}$ since $A$ is a neighbor of $D$, $A$ is a descendant of $D$, $B$ is a descendant of $A$, and $A$ is a collider on $\pi_1$.  
Similarly, the path $\pi_2 =  D \to A \leftarrow C \to B \to E \leftarrow G \to F $ is a DIP with respect to $\mathbb{L}_2 = \{D, A, E, F\}$ because $A$ and $E$ are neighbors of $D$, $A$ is a descendant of $D$, $E$ is a descendant of $A$, $F$ is a descendant of $E$, and both $A$ and $E$ are colliders on $\pi_2$. However, the path $\pi_3= D\to L\leftarrow M\to N\leftarrow O$ is not a DIP because it contains a collider $N \notin Ne(D,\mathcal G)$. Hence, $\pi_3$ can be blocked by conditioning on $\{L\}$, which is a subset of the neighbors of $D$, since the collider does not activate the path toward $O$ through $M$ and $N$, which remains naturally blocked.

$\pi_1$ and $\pi_2$ are also, naturally, classical inducing paths relative to the same subsets $\mathbb L_i$. In contrast, the path $\pi_4 = D \to H \to I \leftarrow K \to J$ is an inducing path relative to $\mathbb O = \{I\}$, but it is not a DIP since $I \notin Ne(D,\mathcal G)$.

Hence, in this example, we have 
\[
DINe(D, \mathcal{G}) = \{B, F\}.
\]

\subsection{Why not all descendant inducing neighbors are spurious neighbors?}

Theorem~\ref{prop:spurious} states that all spurious neighbors are descendant inducing neighbors, but it does not imply equality between these two sets. Indeed, the descendant inducing neighbors of a node $D$ cannot be $d$-separated from $D$ by conditioning on $D$'s neighbors in the DAG $\mathcal{G}$. The subtlety arises from the fact that, during a PC-style skeleton construction, adjacency is assessed at iteration $s$ rather than only considering the true neighbors in $\mathcal{G}$. Consequently, it is possible to remove a spurious neighbor using another node that has not yet been eliminated but is not a true neighbor, making the presence of spurious edges dependent on the size of the conditioning sets used to $d$-separate nodes from the target. The following example illustrates this clearly:

\begin{center}
\begin{minipage}{0.23\textwidth}
\centering
\begin{tikzpicture}[{black, circle, draw, inner sep=0}]
\tikzset{nodes={draw,rounded corners, minimum height=0.6cm, minimum width=0.6cm}}
\node (D)  at (0,0) [fill=gray!30] {$D$};
\node (A)  at (0,-1) {$A$};
\node (B) at (1,-1) {$B$};
\node (C) at (1,0) {$C$};
\node (E) at (0,1) {$E$};
\node (F) at (1,1) {$F$};
\draw[->, >=latex] (D) -- (A);
\draw[->, >=latex] (A) -- (B);
\draw[->, >=latex] (C) -- (B);
\draw[->, >=latex] (C) -- (A);
\draw[->, >=latex] (F) -- (C);
\draw[->, >=latex] (D) -- (F);
\draw[->, >=latex] (E) -- (D);
\draw[->, >=latex] (E) -- (C);
\end{tikzpicture}

$\mathcal G$
\end{minipage}
\hfill
\begin{minipage}{0.23\textwidth}
\centering
\begin{tikzpicture}[{black, circle, draw, inner sep=0}]
\tikzset{nodes={draw,rounded corners, minimum height=0.6cm, minimum width=0.6cm}}
\node (D)  at (0,0) [fill=gray!30] {$D$};
\node (A)  at (0,-1) {$A$};
\node (B) at (1,-1) {$B$};
\node (C) at (1,0) {$C$};
\node (E) at (0,1) {$E$};
\node (F) at (1,1) {$F$};
\draw[-, >=latex] (D) -- (A);
\draw[-, >=latex] (D) -- (B);
\draw[-, >=latex] (D) -- (C);
\draw[-, >=latex] (D) -- (E);
\draw[-, >=latex] (D) -- (F);
\end{tikzpicture}

$s=0$, $s=1$
\end{minipage}
\hfill
\begin{minipage}{0.23\textwidth}
\centering
\begin{tikzpicture}[{black, circle, draw, inner sep=0}]
\tikzset{nodes={draw,rounded corners, minimum height=0.6cm, minimum width=0.6cm}}
\node (D)  at (0,0) [fill=gray!30] {$D$};
\node (A)  at (0,-1) {$A$};
\node (B) at (1,-1) {$B$};
\node (C) at (1,0) {$C$};
\node (E) at (0,1) {$E$};
\node (F) at (1,1) {$F$};
\draw[-, >=latex] (D) -- (A);
\draw[-, >=latex] (D) -- (E);
\draw[-, >=latex] (D) -- (F);
\end{tikzpicture}

$s=2$
\end{minipage}
\end{center}

The difference here compared to Example 1 in the appendix is that $C$ is no longer unconditionally separated from $D$, but is separated conditionally on a set of size 2, $\{E,F\}$. Thus, at the iteration corresponding to conditioning sets of size $s=2$, one can $d$-separate $D$ and $B$ conditionally on $\{A,C\}$ because $C$ could not be separated by a smaller set. Therefore, even though $B$ is a descendant inducing neighbor of $D$, the interplay of conditioning set sizes may result in no spurious edge being present, and the edge between $D$ and $B$ is correctly removed.

\subsection{An example of a LEG containing spurious neighbors}

Figure~\ref{fig:example_with_spurious} provides a complementary example to Figure~\ref{fig:dag_two_legs}, which contained no spurious neighbors for simplicity. Here, we construct the LEGs centered on $Y$ for hop 0, 1, and 2 from the DAG $\mathcal G$. In $\mathcal L^{Y,0}$, the red edge illustrates the spurious neighbor $Z_4$. Indeed, in a PC-style local procedure centered on $Y$, the edge $Z_3-Y$ would be pruned at the first iteration ($s=0$), preventing any future conditioning on $Z_3$. However, separating $Y$ and $Z_4$ requires conditioning on $\{Z_3, Z_1\}$, making $Z_4$ a spurious neighbor of $Y$.  

In $\mathcal L^{Y,1}$, this spurious edge disappears, as stated in item~1 of Theorem~\ref{theorem:leg}, since $Z_4$ is not in the 1-neighborhood of $Y$. No spurious neighbors appear for hop $h=1$, which is consistent with the absence of DIPs between nodes within and outside the 1-neighborhood. Finally, in $\mathcal L^{Y,2}$, the red edge between $Z_2$ and its spurious neighbor $Z_6$ arises because a local PC-style algorithm would prune $Z_3$ from possible conditioning sets for $Z_2$ at $s=0$, although $Z_3$ is required to separate $Z_2$ and $Z_6$.  

In this last LEG, all adjacencies of $Y$ are oriented, and the CDE is therefore identifiable from hop $h=2$. This example also illustrates that spurious edges occur only between the current neighborhood and the outside, and they are progressively pruned as the hop increases.

\begin{figure}[t]
    \centering
        \begin{minipage}{.24\textwidth}
        \centering
    \begin{subfigure}{}{$\mathcal{G}$.}
        \centering
        
        \begin{tikzpicture}[{black, circle, draw, inner sep=0}]
            \tikzset{nodes={draw,rounded corners, minimum height=0.6cm, minimum width=0.6cm}}

            \node (X)  at (0,0)  {$X$};
            \node (Z2) at (0,-.9) {$Z_2$};
            \node (Z3) at (.9,-.9) {$Z_3$};
            \node (Z4) at (1.8,-.9) {$Z_4$};
            \node (Z5) at (0,-1.8) {$Z_5$};
            \node (Z6) at (.9,-1.8) {$Z_6$};
            \node (Z7) at (1.8,-1.8) {$Z_7$};
            \node (Y)  at (.9,0) [fill=red!30] {$Y$};
            \node (Z1) at (1.8,0) {$Z_1$};

            \draw[->, >=latex] (X) -- (Y);
            \draw[->, >=latex] (Y) -- (Z1);
            \draw[->, >=latex] (Z2) -- (X);
            \draw[->, >=latex] (Z3) -- (X);
            \draw[->, >=latex] (Z2) -- (Z5);
            \draw[->, >=latex] (Z1) -- (Z4);
            \draw[->, >=latex] (Z3) -- (Z1);
            \draw[->, >=latex] (Z3) -- (Z4);
            \draw[->, >=latex] (Z3) -- (Z6);
            \draw[->, >=latex] (Z3) -- (Z5);
            \draw[->, >=latex] (Z7) -- (Z4);
            \draw[->, >=latex] (Z5) -- (Z6);
        \end{tikzpicture}
        \label{fig:example_with_spurious:dag}
    \end{subfigure}
        \end{minipage}
\hfill
        \begin{minipage}{.24\textwidth}
        \centering
        \begin{subfigure}{}{$\mathcal{L}^{Y,0}$.}
        \centering
        
        \begin{tikzpicture}[{black, circle, draw, inner sep=0}]
            \tikzset{nodes={draw,rounded corners, minimum height=0.6cm, minimum width=0.6cm}}

            \node (X)  at (0,0)  {$X$};
            \node (Z2) at (0,-.9) {$Z_2$};
            \node (Z3) at (.9,-.9) {$Z_3$};
            \node (Z4) at (1.8,-.9) {$Z_4$};
            \node (Z5) at (0,-1.8) {$Z_5$};
            \node (Z6) at (.9,-1.8) {$Z_6$};
            \node (Z7) at (1.8,-1.8) {$Z_7$};
            \node (Y)  at (.9,0) [fill=red!30] {$Y$};
            \node (Z1) at (1.8,0) {$Z_1$};

            \draw[-, >=latex] (X) -- (Y);
            \draw[-, >=latex] (Y) -- (Z1);
            \draw[-,red, >=latex] (Y) -- (Z4);
        \end{tikzpicture}
        \label{fig:example_with_spurious:leg1}
    \end{subfigure}
        \end{minipage}
\hfill
        \begin{minipage}{.24\textwidth}
        \centering
    \begin{subfigure}{}{$\mathcal{L}^{Y,1}$.}
        \centering
        
        \begin{tikzpicture}[{black, circle, draw, inner sep=0}]
            \tikzset{nodes={draw,rounded corners, minimum height=0.6cm, minimum width=0.6cm}}

            \node (X) [fill=gray!30] at (0,0)  {$X$};
            \node (Z2) at (0,-.9) {$Z_2$};
            \node (Z3) at (.9,-.9) {$Z_3$};
            \node (Z4) at (1.8,-.9) {$Z_4$};
            \node (Z5) at (0,-1.8) {$Z_5$};
            \node (Z6) at (.9,-1.8) {$Z_6$};
            \node (Z7) at (1.8,-1.8) {$Z_7$};
            \node (Y)  at (.9,0) [fill=red!30] {$Y$};
            \node (Z1) [fill=gray!30] at (1.8,0) {$Z_1$};

            \draw[-, >=latex] (X) -- (Y);
            \draw[-, >=latex] (Y) -- (Z1);
            \draw[-, >=latex] (Z2) -- (X);
            \draw[-, >=latex] (Z3) -- (X);
            \draw[-||_||, >=latex] (Z1) -- (Z4);
            \draw[-, >=latex] (Z3) -- (Z1);
        \end{tikzpicture}
        \label{fig:example_with_spurious:leg2}
    \end{subfigure}
        \end{minipage}
    \hfill
    \begin{minipage}{.24\textwidth}
        \centering
    \begin{subfigure}{}{$\mathcal{L}^{Y,2}$.}
        \centering
        
        \begin{tikzpicture}[{black, circle, draw, inner sep=0}]
            \tikzset{nodes={draw,rounded corners, minimum height=0.6cm, minimum width=0.6cm}}

            \node (X) [fill = gray!30]  at (0,0)  {$X$};
            \node (Z2)[fill = gray!30] at (0,-.9) {$Z_2$};
            \node (Z3) [fill = gray!30]at (.9,-.9) {$Z_3$};
            \node (Z4) [fill = gray!30] at (1.8,-.9) {$Z_4$};
            \node (Z5) at (0,-1.8) {$Z_5$};
            \node (Z6) at (.9,-1.8) {$Z_6$};
            \node (Z7) at (1.8,-1.8) {$Z_7$};
            \node (Y)  at (.9,0) [fill=red!30] {$Y$};
            \node (Z1) [fill = gray!30] at (1.8,0) {$Z_1$};

            \draw[->, >=latex] (X) -- (Y);
            \draw[->, >=latex] (Y) -- (Z1);
            \draw[->, >=latex] (Z2) -- (X);
            \draw[->, >=latex] (Z3) -- (X);
            \draw[-, >=latex] (Z2) -- (Z5);
            \draw[->, >=latex] (Z1) -- (Z4);
            \draw[->, >=latex] (Z3) -- (Z1);
            \draw[-, >=latex] (Z3) -- (Z4);
            \draw[-||_||, >=latex] (Z3) -- (Z6);
            \draw[-, >=latex] (Z3) -- (Z5);
             \draw[-,red, >=latex] (Z2) -- (Z6);
            \draw[-, >=latex] (Z7) -- (Z4);
        \end{tikzpicture}
        \label{fig:example_with_spurious:leg3}
    \end{subfigure}
    \end{minipage}
   \caption{A DAG $\mathcal{G}$ and the LEGs $\mathcal{L}^{Y,0}$, $\mathcal{L}^{Y,1}$, and $\mathcal{L}^{Y,2}$ around node $Y$. Red: outcome/target $Y$; blue: treatment $X$; grey: $h$-neighborhood nodes; red arrow: direct effect ($M$: mediator).}
    \label{fig:example_with_spurious}
\end{figure}

\section{Pseudo-codes and background knowledge}
\label{appendix:pseudocode}

\begin{algorithm}[t!]
\caption{\textbf{LocPC} \quad \textcolor{blue}{(LocPC-BK)}}
\label{algo:LocPC}
\begin{algorithmic}[1]
    \REQUIRE Variables $\mathbb{V}$, target node $Y$, hop $h$, known sepsets $\mathcal{S}$, \textcolor{blue}{Forbidden descendants $\overline {\mathcal D}$}
    \STATE $\widehat{\mathcal{L}} \leftarrow$ fully disconnected graph on $\mathbb{V}$
    \STATE $\mathbb{D} = \{Y\}, \quad \text{visited} = \{Y\}, \quad k = 0$
    
    \WHILE{$k \leq h$}
        \FOR{$D \in \mathbb{D}$}
            \FOR{$B \in \mathbb{V} \setminus \{D\}$ s.t. $\mathcal S(B,D) = \texttt{None}$}
                \STATE add edge $D-B$ to $\widehat{\mathcal{L}}$ 
            \ENDFOR
        \ENDFOR

        \STATE $s = 0,\quad $ 
        \WHILE{$\exists D\in \mathbb D$ such that $|adj(D)|-1 \ge s$}
        \STATE \textbf{Store} $adj(D) \gets Ne(D, \widehat{\mathcal{L}}), \; \forall D \in \mathbb{D}$

            \FOR{$D \in \mathbb{D}$}
                \STATE $\text{visited} = \text{visited} \cup \{D\}$
                \FOR{$B \in Ne(D, \widehat{\mathcal{L}})$ \textcolor{blue}{\textbf{and} $\neg [(B\in \text{visited})\land (D\in \overline{\mathcal D}(B))]$}}
                    \IF{$|adj(D) \setminus \{B\}| \ge s$ }
                        \FORALL{$\mathbb{S} \subseteq adj(D) \setminus \{B\}$ \textbf{with} $|\mathbb{S}| = s$}
                            \IF{$D \indep B \mid \mathbb{S}$}
                                \STATE Update $\mathcal{S}$ with $\mathcal S(D, B) = \mathbb{S}$
                                \STATE remove edge $D-B$ from $\widehat{\mathcal{L}}$; \textbf{break}
                            \ENDIF
                        \ENDFOR
                    \ENDIF
                \ENDFOR

            \ENDFOR
            \STATE $ s = s + 1$
        \ENDWHILE

        \STATE $\mathbb{D}_{\text{new}} = \{ Ne(D, \widehat{\mathcal{L}})|D\in \mathbb D\},\quad \mathbb{D} = \mathbb{D}_{new}, \quad k = k + 1$
    \ENDWHILE

     \FOR{each unshielded triple $A\!-\!B\!-\!C\in NeHood(Y,h,\widehat{\mathcal L})$ with $B \notin \mathcal S(A,C)$} 
     \STATE orient $A \rightarrow B \leftarrow C$
    \ENDFOR

    \STATE \textbf{Apply} Meek rules repeatedly on $\widehat{\mathcal{L}}$

     \STATE \textbf{Apply} CI-NNC rule (Def.~\ref{def:local_rules}) repeatedly on $\widehat{\mathcal{L}}$


    \RETURN $\widehat{\mathcal{L}}, \mathcal{S}, \text{visited}$
\end{algorithmic}
\end{algorithm}

\subsection{LocPC}

Algorithm~\ref{algo:LocPC} describes the \textbf{LocPC} procedure for learning a local essential graph (LEG) in the $h$-hop neighborhood of a target node $Y$. A specific form of background knowledge (BK), distinct from the usual one involving forbidden or mandatory edges and orientations, can also be incorporated into \textbf{LocPC}  (as described in the paragraph following the algorithm explanation); it corresponds to the portions of the pseudocode highlighted in blue.

LocPC begins by initializing the estimated LEG $\widehat{\mathcal{L}}$ as a fully disconnected graph over $\mathbf{V}$ (line~1), setting the exploration frontier $\mathbf{D} = \{Y\}$, the visited list to $\{Y\}$, and the hop counter $k=0$ (line~2). The first phase (\emph{local skeleton discovery}, lines~3--19) expands $\mathbf{D}$ while $k \le h$. At each hop, $\mathbf{D}_{\text{new}}$ is reset, and all unexplored nodes are temporarily connected to $\mathbf{D}$ (lines~4--6). 
CI tests are then performed between each $D \in \mathbf{D}$ and its neighbors, using increasing conditioning set sizes $s$ (line~7). 
If $D \indep B \mid \mathbf{S}$, the edge $D-B$ is removed and $\mathbf{S}$ recorded (lines~16--17). 
New neighbors are added to $\mathbf{D}_{\text{new}}$ for the next hop, and $k$ is incremented (line~18--19). In the second phase (\emph{orientation}, lines~20--24), unshielded colliders $A \to B \leftarrow C$ are oriented if $B \notin \text{Sepset}(A,C)$ (lines~20--21). The three Meek rules are then applied locally (line~22), followed by the NNC rule (lines~23). The algorithm returns the estimated LEG $\widehat{\mathcal{L}}$, the separating sets $\mathcal{S}$, and the visited nodes, which can be reused in downstream procedures such as \textbf{LocPC-CDE}. Parts highlighted in blue refers to the incorporation of background knowledge.

\begin{algorithm}[t!]
    \caption{\textbf{LocPC-CDE} $\quad$ \textcolor{blue}{(LocPC-CDE-BK : replace \textbf{LocPC} by \textbf{LocPC-BK})}}
    \label{algo:LocPC-CDE}
    \begin{algorithmic}[1]
        \REQUIRE Variables $\mathbb V$, treatment $X$, outcome $Y$, known sepsets $\mathcal S_0$
        \STATE $\widehat{\mathcal{L}}, \mathcal S,\text{visited}= \text{\textbf{LocPC}}(\mathbb{V},Y,h=0,\mathcal S_0)$ 
        \STATE $\mathbb D= [Y]$
        \STATE $ h= 1$
        \WHILE{$X\in Ne(Y,\widehat{\mathcal{L})}$ and $Y\notin Pa(X,\widehat{\mathcal{L}})$ and $Ne_{\mathrm{nar}}(Y,\widehat{\mathcal{L}}) \ne \emptyset$ and $\text{visited}\ne \mathbb V$}

        \STATE $\widehat{\mathcal{L}},\mathcal S,\text{visited}= \text{\textbf{LocPC}}(\mathbb V,Y,h,\mathcal S)$


        \FOR{$D\in \mathbb D$}
            \STATE $\mathbb D= \mathbb D\cup \{Ne_{\mathrm{nar}}(D,\widehat{\mathcal{L}})\cap NeHood(Y,h,\widehat{\mathcal L})\}$
        \ENDFOR
        \IF{$\mathbb D$ satisfies the non-orientability criterion (Def.~\ref{def:orientability})}
        \STATE \textbf{break}
        \ENDIF
        \STATE $h= h+1$
        \ENDWHILE
    \IF{$X\in Ne(Y,\widehat{\mathcal{L}})$ and $Y\notin Pa(X,\widehat{\mathcal{L}})$ and $Ne_{\mathrm{nar}}(Y,\widehat{\mathcal{L}}) \ne \emptyset$}
        \STATE $\text{identifiable}= False$
        \ELSE 
        \STATE $\text{identifiable}= True$
        \ENDIF
    
    \RETURN $\text{identifiable}, \widehat{\mathcal{L}}$
    \end{algorithmic}
\end{algorithm}

\subsection{LocPC-CDE}

First, we define the set of non-arrow neighbors of a node $Y$ in a graph $\mathcal{G}_0 = (\mathbb {V}_0, \mathbb{E}_0)$ as
\[
Ne_{\mathrm{nar}}(Y, \mathcal {G}_0) = Ne(Y, \mathcal {G}_0) \setminus \{V \in \mathbb{V}_0 \mid (Y \to V) \in \mathbb{E}_0 \text{ or } (Y \leftarrow V) \in \mathbb{E}_0 \}.
\]

Algorithm~\ref{algo:LocPC-CDE} uses \textbf{LocPC} to iteratively explore the neighborhood of $Y$ and determine CDE identifiability. 
It starts by discovering the 0-hop LEG around $Y$ (line~1), storing separating sets and visited nodes, with the visited set initialized to $\{Y\}$. 
The candidate set $\mathbf{D}$ for testing non-orientability is also initialized to $\{Y\}$, and the hop counter $h$ is set to 1 (line~3). Exploration continues while there exists a node $X$ adjacent to $Y$ that is not a child of $Y$, unoriented edges remain incident to $Y$, and not all nodes have been discovered (line~4). 
At each iteration, the $h$-hop LEG is computed, updating previously stored separating sets $\mathcal{S}$ and visited nodes (line~5). 
Nodes with unoriented edges connected to $\mathbf{D}$ within the $h$-hop neighborhood are added to $\mathbf{D}$ (lines~6–7), ensuring $\mathbf{D}$ always contains $Y$ and candidates for the non-orientability criterion. 
The algorithm checks if $\mathbf{D}$ satisfies the non-orientability criterion (line~8); if so, it stops (line~9). 
Otherwise, $h$ is incremented and exploration continues (line~10). Upon termination, two outcomes are possible: 
(1) if $X$ remains adjacent to $Y$ with unoriented edges, $CDE(x,x',Y)$ is not identifiable (lines~11–12); 
(2) otherwise, the CDE is identifiable (lines~13–14). The algorithm returns identifiability and the LEG; estimation can be added when the CDE is identifiable, e.g., via regression on the estimated parents of $Y$ in a linear setting. Parts highlighted in blue refers to the incorporation of background knowledge.

\subsection{Incorporating background knowledge}

The background knowledge can be incorporated into \textbf{LocPC}  to improve computational efficiency. Specifically, edges and orientations can be specified a priori in the form of forbidden/mandatory edges and forbidden orientations. These are handled as in any constraint-based causal discovery algorithm that integrates background knowledge: the forced edges and orientations are added beforehand and take precedence over those inferred by the algorithm. We therefore do not elaborate further on this standard mechanism.
The more interesting aspect of background knowledge lies in leveraging our results on spurious neighbors and descendant inducing paths by introducing a forbidden descendant set, denoted $\overline{\mathcal D}$. For each node $B$, $D \in \overline{\mathcal D}(B)$ if the user enforces that $D$ is a non-descendant of $B$. Since any spurious edge added during the procedure must necessarily point toward a descendant (Theorem~\ref{prop:spurious}), if node $B$ has already been visited and did not prune the edge between $B$ and $D$, then—because $D$ is a non-descendant of $B$—this edge cannot be spurious. Consequently, it is unnecessary to retest the edge between $B$ and $D$ when the algorithm later visits $D$, thereby avoiding redundant independence tests. Moreover, the edge is automatically oriented as $D \to B$ to ensure consistency with the background knowledge.
This specific modification is highlighted in blue in the following pseudocodes.

\section{Experiments}
\label{appendix:experiments}

The complete, reproducible code used for the experiments is available at: \url{github.com/CIPHOD/pyCIPHOD/tree/main/reproducibility/clear2026}.
We used our implementations of \textbf{LocPC-CDE} and PC~\citep{Spirtes_2000} algorithms, available at: \url{github.com/CIPHOD/pyCIPHOD}.
We use the publicly available implementation of LDECC algorithm from~\citep{Gupta_2023}, accessible at \url{github.com/acmi-lab/local-causal-discovery}, and the implementation of the CMB algorithm~\citep{Gao_2015} and MBbyMB algorithm~\citep{wang2014discovering} available at \url{github.com/wt-hu/pyCausalFS}. Some algorithms were minimally adapted to produce the same evaluation metrics as \textbf{LocPC-CDE}.

\subsection{Synthetic Data}

We detail here the procedure used to generate the graphs and simulate the data, as well as how the evaluation metrics are computed. 

\paragraph{DAGs generation}

All random graph models considered are homogeneous Erdős–Rényi models with edge existence probability \( p = \frac{2}{|\mathbb{V}| - 1} \). This ensures constant sparsity as the number of variables varies (on average, each node in the graph is adjacent to 2 edges).

For the \textbf{Identifiable CDE Case}, for each number of variables $|\mathbb{V}|$, we generate a DAG as follows:

\begin{enumerate}
    \item Generate an undirected Erdős–Rényi graph, where each edge exists independently with probability $p$. Then, sample a random permutation $\sigma$ of $\{1, \cdots, |\mathbb{V}|\}$ to define a topological (causal) order. For each undirected edge $i-j$, if $\sigma(i) < \sigma(j)$, orient the edge as $i \to j$. This results in a DAG whose sparsity is controlled by the edge probability $p$.

    \item Convert the resulting DAG into its essential graph. Search for a pair of variables $(X, Y)$ such that (i) $X \to Y$ is in the DAG (to ensure the existence of a direct effect) and (ii) all adjacents of $Y$ are oriented in the essential graph. This guarantees identifiability of  $CDE(x,x',y)$, according to Theorem 5.4 of~\citep{Flanagan_2020}.

    \item If no such pair $(X,Y)$ exists, generate a new random graph and repeat until the condition is met. The final graph thus guarantees that the direct effect from $X$ to $Y$ is identifiable.
\end{enumerate}



For the \textbf{Non-Identifiable CDE} case, the procedure is similar, with a modified condition to ensure non-identifiability:

\begin{enumerate}
    \item Generate an undirected Erdős–Rényi graph and orient it according to a random topological order $\sigma$, as described above.

    \item Convert the DAG to its essential graph and look for a pair of variables $(X,Y)$ such that (i) $X \to Y$ is present in the DAG, and (ii) at least one adjacent edge to $Y$ remains unoriented in the essential graph. This guarantees that $CDE(x,x',y)$ is \textit{not} identifiable, according to Theorem 5.4 of~\citep{Flanagan_2020}.

    \item If no such pair $(X,Y)$ exists, repeat the process until one is found.
\end{enumerate}

The data is then simulated using the linear/non-linear SCM procedure described below.

\paragraph{Data simulation}

Let $\mathcal{G}$ denote the causal structure. 

A \textbf{linear Gaussian} SCM can be expressed, 
can be written in matrix form as: 
    $\mathbb{V} = B \mathbb{V} + \bm{\xi},$
 where $B$ is a coefficient matrix that can be permuted to lower-triangular form (due to the causal ordering) and $\bm{\xi} \sim \mathcal{N}(0, \bm{\Sigma})$ is a Gaussian noise vector of dimension $|\mathbb{V}|$. The solution is then given by: $\mathbb{V} = (I - B)^{-1} \bm{\xi}.$

The simulation procedure is as follows:
\begin{enumerate}
    \item Generate a random lower-triangular coefficient matrix $B$, with non-zero entries sampled uniformly from $\{ x \in [-1, 1] : |x| > 0.2 \}$;
    \item Sample $5000$ independent noise vectors $\bm{\xi}$, with each component $\xi_j \sim \mathcal{N}(0, \sigma_j^2)$ and $\sigma_j^2 \sim \mathcal{U}[0.8, 1]$;
    \item For each noise vector, compute $\mathbb{V} = (I - B)^{-1} \bm{\xi}$, resulting in $5000$ independent observations from the linear SCM.
\end{enumerate}

For the \textbf{nonlinear} case, we simulate binary variables to model categorical data commonly encountered in practice. Let $\mathcal{G}$ be the causal DAG. For each variable $V_i$, $i=1,\ldots,|\mathbb{V}|$, the binary variable is generated as:  $V_i = \mathbb{I}_{\xi_i \leq p_i},$
where $\mathbb{I}$ is the indicator function, $\xi_i \sim \mathcal{U}([0,1])$ is a uniform random variable, and  
    $p_i = \left({1 + \exp\left(-\sum_{V_j \in Pa(V_i, \mathcal{G})} a_{j,i} V_j\right)}\right)^{-1}.$
The coefficients $a_{j,i}$ are sampled uniformly from $\{ x \in [-5,5] : |x| > 0.2 \}$. Then, $5000$ independent observations are generated by first sampling vectors $\xi_i \sim \mathcal{U}([0,1])$ of size $5000$ and simulating the variables $V_i$ following the causal ordering.  


\paragraph{Estimation and Evaluation Metrics}

For each method and each graph, we apply the causal discovery algorithm, which outputs either a fully or partially oriented causal graph. We evaluate the output based on the following criteria: \textbf{(1) Identifiability detection:} Whether the method correctly determines if all adjacents of $Y$ are oriented (in the identifiable case) or not (in the non-identifiable case); \textbf{(2) Parent recovery (identifiable case only):} If all adjacents of $Y$ are oriented, we compare the set of estimated parents of $Y$ to the true parents; \textbf{(3) Number of CI tests:} The number of CI tests performed by each method is recorded.

The proportion of correctly identified non-identifiable graphs, shown in Figure~\ref{fig:exp}, is computed as the ratio of graphs for which the method returns a non-identifiable output to the total number of graphs (100) for each value of $|\mathbb{V}|$.

The F1 score is computed to evaluate the accuracy of parent recovery in identifiable cases. Let $\widehat{Pa(Y)}$ be the set estimated by the method and $Pa(Y,\mathcal G)$ be the set of true parents. Then:
\begin{equation*}
    \text{Precision} = \frac{|\widehat{Pa(Y)} \cap Pa(Y,\mathcal G)|}{|\widehat{Pa(Y)}|}, \quad \text{Recall} = \frac{|\widehat{Pa(Y)} \cap Pa(Y,\mathcal G)|}{|Pa(Y,\mathcal G)|}, \quad F_1 = 2 \cdot \frac{\text{Precision} \cdot \text{Recall}}{\text{Precision} + \text{Recall}}.
\end{equation*}

\subsection{Real Data}

Causal discovery was performed on aggregated data from the French National Health Data System (SNDS), covering the incidence of various pathologies across 101 departments in 2023. The dataset used is available in the open data from Ameli, accessed February 2026, \url{data.ameli.fr/explore/dataset/effectifs/export/?refine.annee=2023&refine.cla_age_5=tsage&refine.sexe=9&refine.region=99&refine.niveau_prioritaire=1}. 

\end{document}